\documentclass{article}

\usepackage[final]{neurips_2019}
\usepackage[utf8]{inputenc} % allow utf-8 input
\usepackage[T1]{fontenc}    % use 8-bit T1 fonts
\usepackage{hyperref}       % hyperlinks
\usepackage{url}            % simple URL typesetting
\usepackage{booktabs}       % professional-quality tables
\usepackage{amsfonts}       % blackboard math symbols
\usepackage{nicefrac}       % compact symbols for 1/2, etc.
\usepackage{microtype}      % microtypography

\usepackage{graphicx}
\usepackage{subfigure}
\usepackage{tabularx}
\usepackage{algorithm}
\usepackage{algpseudocode}
\usepackage{caption}
\usepackage{amssymb,amsmath,amsthm}
\usepackage{bm}
\usepackage{amsmath,amsfonts,bm}

\def\eqref#1{equation~\ref{#1}}
\def\1{\bm{1}}

\DeclareMathAlphabet{\mathsfit}{\encodingdefault}{\sfdefault}{m}{sl}
\SetMathAlphabet{\mathsfit}{bold}{\encodingdefault}{\sfdefault}{bx}{n}

\newcommand{\E}{\mathbb{E}}

\DeclareMathOperator*{\argmax}{arg\,max}

\DeclareMathOperator{\sign}{sign}

\DeclareMathOperator*{\where}{where}

\DeclareMathOperator*{\y}{y}

\DeclareMathOperator*{\diag}{diag}

\newtheorem{theorem}{Theorem}
\newtheorem{corollary}{Corollary}
\newtheorem{lemma}{Lemma}
\newtheorem{proposition}{Proposition}
\newtheorem{assumption}{Assumption}

\edef\oldassumption{\the\numexpr\value{assumption}+1}

\title{An Adaptive Empirical  Bayesian Method for Sparse Deep Learning}

\author{%
  Wei Deng \\
  Department of Mathematics\\
  Purdue University\\
  West Lafayette, IN 47907 \\
  \texttt{deng106@purdue.edu} \\
  \And
  Xiao Zhang \\
  Department of Computer Science \\
  Purdue University \\
  West Lafayette, IN 47907 \\
  \texttt{zhang923@purdue.edu} \\
  \And
  Faming Liang \\
  Department of Statistics \\
  Purdue University \\
  West Lafayette, IN 47907 \\
  \texttt{fmliang@purdue.edu} \\
  \And
  Guang Lin \\
  Departments of Mathematics, Statistics \\
  and School of Mechanical Engineering \\
  Purdue University \\
  West Lafayette, IN 47907 \\
  \texttt{guanglin@purdue.edu} \\
}

\begin{document}

\maketitle

\begin{abstract}
We propose a novel adaptive empirical Bayesian (AEB) method for sparse deep learning, where the sparsity is ensured via a class of self-adaptive spike-and-slab priors. The proposed method works by alternatively sampling from an adaptive hierarchical posterior distribution using stochastic gradient Markov Chain Monte Carlo (MCMC) and smoothly optimizing the hyperparameters using stochastic approximation (SA). We further prove the convergence of the proposed method to the asymptotically correct distribution under mild conditions. Empirical applications of the proposed method lead to the state-of-the-art performance on MNIST and Fashion MNIST with shallow convolutional neural networks (CNN) and the state-of-the-art compression performance on CIFAR10 with Residual Networks. The proposed method also improves resistance to adversarial attacks.
\end{abstract}

\section{Introduction}
%Deep learning has achieved many successes in various areas. However, researchers have largely focused on increasing its learning ability by merely applying point estimation without considering uncertainty quantification. As discussed by \citet{Kendall17, Gal17,Amodei16}, the AI safety problems can never be ignored: Reported by US News, an image classification system confused humans with gorillas; A blog article by Keras developers showed a well-trained ImageNet model predicted adversarial sea snake images as magpies with 99.99\% confidence; Several serious traffic fatalities caused by autonomous cars were reported recently, resulting from the intelligent systems for assisted driving failing to recognize objects. These issues, which weakened the reliability of point estimation in deep learning, can be considerably alleviated given a better uncertainty quantification of the decisions.

%%%
% A brief intro of bayesian deep learning and SG-MCMC
%%%
% Bayesian deep learning, which evolved from Bayesian neural networks \citep{Neal96,Denker90}, provides an alternative to point estimation due to its close connection to both Bayesian probability theory and cutting-edge deep learning models. Its merit of quantifing uncertainty has been studied \citep{Gal16b}, which not only increases the predictive power of deep neural networks (DNNs) but also further provides a more robust estimation to enhance AI safety. 

MCMC, known for its asymptotic properties, has not been fully investigated in deep neural networks (DNNs) due to its unscalability in dealing with big data. Stochastic gradient Langevin dynamics (SGLD) \citep{Welling11}, the first stochastic gradient MCMC (SG-MCMC) algorithm, tackled this issue by adding noise to the stochastic gradient, smoothing the transition between optimization and sampling and making MCMC scalable. \citet{Chen14} proposed using stochastic gradient Hamiltonian Monte Carlo (SGHMC), the second-order SG-MCMC, which was shown to converge faster. In addition to modeling uncertainty, SG-MCMC also has remarkable non-convex optimization abilities. \citet{Maxim17,Xu18} proved that SGLD, the first-order SG-MCMC, is guaranteed to converge to an approximate global minimum of the empirical risk in finite time. \citet{Yuchen17} showed that SGLD hits the approximate local minimum of the population risk in polynomial time. \citet{Mangoubi18} further demonstrated SGLD with simulated annealing has a higher chance to obtain the global minima on a wider class of non-convex functions. However, all the analyses fail when DNN has too many parameters, and %compared to the amount of data. %One main reason that causes that over-fitting problem is the over-specified models. 
% When this happens, 
the over-specified model tends to have a large prediction variance, resulting in poor generalization and causing over-fitting. Therefore, a proper model selection is on demand at this situation.% \citep{Chen15}. % \citet{Saatci17} used SGHMC with GANs \citep{Goodfellow14b} to achieve a fully probabilistic inference and showed the Bayesian GAN model with only 100 labeled images was able to achieve 99.3\% testing accuracy in MNIST dataset. 
A standard method to deal with model selection is variable selection. Notably, the best variable selection based on the $L_0$ penalty is conceptually ideal for sparsity detection but is computationally slow. Two alternatives emerged to approximate it. On the one hand, penalized likelihood approaches, such as Lasso \citep{Tibshirani94}, induce sparsity due to the geometry that underlies the $L_1$ penalty. To better handle highly correlated variables, Elastic Net was proposed \citep{Zou05} and makes a compromise between $L_1$ penalty and $L_2$ penalty. On the other hand, spike-and-slab approaches to Bayesian variable selection originates from probabilistic considerations. \citet{George93} proposed to build a continuous approximation of the spike-and-slab prior to sample from a hierarchical Bayesian model using Gibbs sampling. This continuous relaxation inspired the efficient EM variable selection (EMVS) algorithm in linear models \citep{Rockova14, Rockova18}. %A similar algorithm \citep{Rockova18} using Spike-and-slab LASSO (SSL) priors is presented later to bridge the gap between variable selection and penalized likelihood estimation. 

%%%
% Advances of model selection in DNN
%%%
Despite the advances of model selection in linear systems, model selection in DNNs has received less attention. %A few attempts worked well on shallow Bayesian neural networks (BNNs): \citet{liang_vs_2018} % proposed to use parallel MCMC \citep{song14} to sample from BNNs with Gaussian priors on the weights and
%analyzed the consistency properties of variable selection using Gaussian priors on BNNs. 
\citet{Ghosh18} proposed to use variational inference (VI) based on regularized horseshoe priors to obtain a compact model. \citet{liang_vs_2018} presented the theory of posterior consistency for Bayesian neural networks (BNNs) with Gaussian priors, and \citet{ye18} applied a greedy elimination algorithm to conduct group model selection with the group Lasso penalty. Although these works only show the performance of shallow BNNs, the experimental methodologies imply the potential of model selection in DNNs. 
\citet{Louizos17} studied scale mixtures of Gaussian priors and half-Cauchy scale priors for the hidden units of VGG models \citep{Karen14} 
and achieved good model compression performance on CIFAR10 \citep{Alex09} using VI. However, due to the limitation of VI in non-convex optimization, the compression is still not sparse enough and can be further optimized.  %Variable selection in much more complex DNNs is still an open question.%\citep{Aidan18} proposed the targeted dropout method to rank weights or units according to the magnitude and prune the elements with low importance. %They achieved xx\% on CIFAR10 using Resnet32. 
% the proven success of $L_1$ penalty in variable selection, and the wide application of $L_2$ penalty in DNNs

%%%
% Our method, to apply model selection in DNN
%%%
%which is impractical in real applications for smart devices.
Over-parameterized DNNs often demand for tremendous memory use and heavy computational resources, which is impractical for smart devices. More critically, over-parametrization frequently overfits the data and results in worse performance \citep{lin17}. To ensure the efficiency of the sparse sampling algorithm without over-shrinkage in DNN models, we propose an AEB method to adaptively sample from a hierarchical Bayesian DNN model with spike-and-slab Gaussian-Laplace (SSGL) priors and the priors are learned through optimization instead of sampling. The AEB method differs from the full Bayesian method in that the priors are inferred from the empirical data and the uncertainty of the priors is no longer considered to speed up the inference. In order to optimize the latent variables without affecting the convergence to the asymptotically correct distribution, stochastic approximation (SA) \citep{Albert90}, a standard method for adaptive sampling \citep{andrieu05, Liang10}, naturally fits to train the adaptive hierarchical Bayesian model. %As a result, the asymptotic property allows us to combine simulated annealing to obtain a better point estimate in non-convex optimization.
%EM is an efficient algorithm to deal with this type of priors in linear models with closed-form update. Unfortunately, EM is known to be sensitive to hyperparameters, vulnerable to local traps \citet{ma18}, and unable to .  
%showed theoretically EM can be quite slow in searching global minima in Gaussian mixture models given bad initializations. 
%To tackle this issue, our solution is to use 
%\citet{Rockova17} alleviates the local entrapment problem by deploying an ensemble of interactive repulsive particles. 

% Motivated by the computational efficiency of the EM algorithm \citep{Dempster77}, the proven success of $L_1$ penalty in variable selection, and the wide application of $L_2$ penalty (known as weight decay in DNNs) in parameter estimation, 

%%%
% List our contributions
%%%
In this paper, we propose a sparse Bayesian deep learning algorithm, SG-MCMC-SA, to adaptively learn the hierarchical Bayes mixture models in DNNs. This algorithm has four main contributions: 
\begin{itemize}
\item We propose a novel AEB method to efficiently train hierarchical Bayesian mixture DNN models, where the parameters are learned through sampling while the priors are learned through optimization.

\item We prove the convergence of this approach to the asymptotically correct distribution, and it can be further generalized to a class of adaptive sampling algorithms for estimating state-space models in deep learning.

\item We apply this adaptive sampling algorithm in the DNN compression problems firstly, with potential extension to a variety of model compression problems.
% \item By automatically searching the over-fitted parameter space to add more penalties, the proposed method quantifies the posterior distributions of the decisive variables more accurately;
% \item Empirically, the algorithm demonstrates more resistance to over-fitting in simulations, leads to the state-of-the-art performance on MNIST and Fashion-MNIST and shows more robustness over SG-MCMC.
\item It achieves the state of the art in terms of compression rates, which is 91.68\% accuracy on CIFAR10 using only 27K parameters (90\% sparsity) with Resnet20  \citep{kaiming15}.%, which outperforms the existing methods by a large margin.
\end{itemize}

\section{Stochastic Gradient MCMC}
%%%%%%%%%%%%%%%%%%% Point Estimates of Deep Neural Network %%%%%%%%%%%%%%%%%%
% \subsection{Bayesian Point Estimates of Deep Neural Network}
% Viewing model $P(y|\bm{x}, W)$ probabilistically, where $W$ are the parameters, inputs $\bm{x}$ of a neural network are mapped onto a high-dimensional space by various layers of affine transformation interleaved with nonlinearities. Maximum likelihood estimation (MLE) with priors is equivalent to maximum a posteriori probability (MAP) estimate in Bayesian statistics:
% \begin{equation}
% \begin{split}
% W^{\MAP} & = \argmax_{W} \log P(W|\mathcal{D}) = \argmax_{W} \log P(\mathcal{D}|W) + \log P(W), \nonumber
% \end{split}
% \end{equation}
% where usually $P(W)$ can be given a Gaussian prior (L2 regularization) or a Laplace prior (L1 regularization).

% \subsection{Stochastic Gradient Langevin Dynamics}

We denote the set of model parameters by $\bm{\beta}$, the learning rate at time $k$ by $\epsilon^{(k)}$, the entire data by $\mathcal{D}=\{\bm{d}_i\}_{i=1}^N$, where $\bm{d}_i=(\bm{x}_i, y_i)$, the log of posterior by $L(\bm{\beta})$. The mini-batch of data $\mathcal{B}$ is of size $n$ with indices $\mathcal{S}=\{s_1, s_2, ..., s_n\}$, where $s_i\in \{1, 2, ..., N\}$. Stochastic gradient $\nabla_{\bm{\beta}} \tilde L(\bm{\beta})$ from a mini-batch of data $\mathcal{B}$ randomly sampled from $\mathcal{D}$ is used to approximate $\nabla_{\bm{\beta}} L(\bm{\beta})$:
\begin{equation}
\begin{split}
\label{factor}
\nabla_{\bm{\beta}} \tilde{L}(\bm{\beta})=\nabla_{\bm{\beta}} \log \mathrm{P}(\bm{\beta})+\dfrac{N}{n}\sum_{i\in \mathcal{S}} \nabla_{\bm{\beta}} \log \mathrm{P}(\bm{d}_i|\bm{\beta}).
\end{split}
\end{equation}
SGLD (no momentum) is formulated as follows:
\begin{equation}
\begin{split}
&\bm{\beta}^{(k+1)}=\bm{\beta}^{(k)} + \epsilon^{(k)} \nabla_{\bm{\beta}} \tilde{L}(\bm{\beta}^{(k)})+\mathcal{N}({0, 2\epsilon^{(k)}\tau^{-1}}),\\
\end{split}
\end{equation}
where $\tau>0$ denotes the inverse temperature. It has been shown that SGLD asymptotically converges to a stationary distribution $\pi(\bm{\beta}|\mathcal{D})\propto e^{\tau L(\bm{\beta})}$ \citep{Teh16,Yuchen17}. As $\tau$ increases and $\epsilon$ decreases gradually, the solution tends towards the global optima with a higher probability. Another variant of SG-MCMC, SGHMC \citep{Chen14, yian2015}, proposes to generate samples as follows:
\begin{equation}  
\label{undampedsgld}
\left\{  
             \begin{array}{lr}  
             d \bm{\beta}=\bm{r}dt,  \\  
              & \\
%              \prod_{i=1}^{m} \psi(\bm{x_i};\bm{\beta})^{y_i} \left(1-\psi(x_i;\bm{\beta})\right)^{1-y_i}, &  \text{classification}  
             d\bm{r}=\nabla_{\bm{\beta}} \tilde{L}(\bm{\beta})dt-\bm{C}\bm{r}dt+\mathcal{N}(0, 2\bm{B}\tau^{-1}dt)+\mathcal{N}(0, 2(\bm{C}-\bm{\hat{B}})\tau^{-1}dt),   
             \end{array}  
\right.  
\end{equation} 
where $\bm{r}$ is the momentum item, $\bm{\hat{B}}$ is an estimate of the stochastic gradient variance, $\bm{C}$ is a user-specified friction term. Regarding the discretization of (\ref{undampedsgld}), we follow the numerical method proposed by \citet{Saatci17} due to its convenience to import parameter settings from SGD.

\section{Empirical Bayesian via Stochastic Approximation}
\subsection{A hierarchical formulation with deep SSGL priors}

% \begin{figure*}[t!]
%     \centering
%     \begin{subfigure}[t]{0.67\textwidth}
%         \centering
%         \includegraphics[scale=0.3]{figures/dropout_weight_uncertainty_v9.pdf}
%         \caption{Model Structure}
%     \end{subfigure}%
%     ~ 
%     \begin{subfigure}[t]{0.5\textwidth}
%         \centering
%         \includegraphics[scale=0.08]{figures/dropout.pdf}
%         \caption{Dropout}
%     \end{subfigure}
%     \caption{Left: each weight is a fixed value, neurons are randomly dropped, providing an approximation of $2^N$ possible models; Right: weights are assigned to specific distributions, which end up with several promising models}
% \end{figure*}

% \begin{figure*}[t!]
%     \centering 
%     \begin{subfigure}[t]{0.67\textwidth}
%         \centering
%         \includegraphics[scale=0.3]{figures/dropout_weight_uncertainty_v9.pdf}
%         \caption{Model Structure}
%     \end{subfigure}%
%     ~ 
%     \begin{subfigure}[t]{0.5\textwidth}
%         \centering
%         \includegraphics[scale=0.08]{figures/dropout.pdf}
%         \caption{Dropout}
%     \end{subfigure}
%     \caption{Left: each weight is a fixed value, neurons are randomly dropped, providing an approximation of $2^N$ possible models; Right: weights are assigned to specific distributions, which end up with several promising models}
% \end{figure*}

Inspired by the hierarchical Bayesian formulation for sparse inference \citep{George93}, we assume the weight $\bm{\beta}_{lj}$ in sparse layer $l$ with index $j$ follows the SSGL prior
\begin{equation}
\beta_{lj} | \sigma^2, \gamma_{lj} \sim (1-\gamma_{lj}) \mathcal{L}(0, \sigma v_0)+\gamma_{lj} \mathcal{N}(0, \sigma^2 v_1).
\end{equation}
% \begin{equation}
% \beta_{lj} | \sigma^2, \gamma_{lj} \propto \frac{1}{\sigma}\exp\left[{\frac{(1-\gamma_{lj})|\beta_{lj}|}{\sigma v_0} +\frac{\gamma_{lj}\beta_{lj}^2}{\sigma^2 v_1}}\right],
% \end{equation}
where $\gamma_{lj}\in \{0, 1\}$, $\bm{\beta}_l\in\mathbb{R}^{p_l}$,  $\sigma^2\in\mathbb{R}$, $\mathcal{L}(0, \sigma v_0)$ denotes a Laplace distribution with mean $0$ and scale $\sigma v_0$, and $\mathcal{N}(0, \sigma^2 v_1)$ denotes a Normal distribution with mean $0$ and variance $\sigma^2 v_1$. The sparse layer can be the fully connected layers (FC) in a shallow CNN or Convolutional layers in ResNet. If we have $\gamma_{lj}=0$, the prior behaves like Lasso, which leads to a shrinkage effect; when $\gamma_{lj}=1$, the $L_2$ penalty dominates. The likelihood follows
% http://web.engr.oregonstate.edu/~xfern/classes/cs534/notes/logistic-regression-note.pdf
\begin{equation}  
\label{eq:4}
\pi(\mathcal{B}|\bm{\beta},\sigma^2)=\left\{  
             \begin{array}{lr}  
             \dfrac{\exp\left\{-\dfrac{\sum_{i\in \mathcal{S}} (y_{i}-\psi(\bm{x}_{i};\bm{\beta}))^2}{2\sigma^2}\right\}}{(2\pi\sigma^2)^{n/2}} & \text{(regression)}, \\  
              & \\
%              \prod_{i=1}^{m} \psi(\bm{x_i};\bm{\beta})^{y_i} \left(1-\psi(x_i;\bm{\beta})\right)^{1-y_i}, &  \text{classification}  
             \prod\limits_{i\in \mathcal{S}} \dfrac{\exp\{\psi_{y_i}(\bm{x}_{i};\bm{\beta})\}}{\sum_{t=1}^K \exp\{\psi_{t}(\bm{x}_i;\bm{\beta})\}} & \text{(classification)}, 
             \end{array}
\right.  
\end{equation} 
where $\psi(\bm{x}_{i};\bm{\beta})$ is a linear or non-linear mapping, and $y_i \in \{1,2,\dots,K\}$ is the response value of the $i$-th example. In addition, the variance $\sigma^2$ follows an inverse gamma prior $\pi(\sigma^2)=IG(\nu/2, \nu\lambda/2)$.
The i.i.d. Bernoulli prior is used for $\bm{\gamma}$, namely $\pi(\bm{\gamma}_l|\delta_l) = \delta_l^{|\bm{\gamma}_l|} (1-\delta_l)^{p_l-|\bm{\gamma}_l|}$
where $\delta_l\in\mathbb{R}$ follows Beta distribution $\pi(\delta_l) \propto \delta_l^{a-1}(1-\delta_l)^{b-1}$. The use of self-adaptive penalty enables the model to learn the level of sparsity automatically. 
Finally, our posterior follows
\begin{equation}
\pi(\bm{\beta}, \sigma^2, \delta, \bm{\gamma}|\mathcal{B}) \propto \pi(\mathcal{B}|\bm{\beta}, \sigma^2)^{\frac{N}{n}} \pi(\bm{\beta}|\sigma^2,\bm{\gamma}) \pi(\sigma^2| \bm{\gamma}) \pi(\bm{\gamma}|\delta) \pi(\delta).
\end{equation}

% \begin{equation}  
% \label{eq:4}
% \pi(\mathcal{B}|\bm{\beta}, \sigma^2) = \prod\limits_{i\in \mathcal{S}} \dfrac{\exp\{\psi_{y_i}(\bm{x}_{i};\bm{\beta})\}}{\sum_{t=1}^K \exp\{\psi_{t}(\bm{x}_i;\bm{\beta})\}}
% \end{equation}

% \begin{equation}  
% \label{eq:4}
% \pi(\mathcal{B}|\bm{\beta}, \sigma^2) =\left\{  
%              \begin{array}{lr}  
%              (2\pi\sigma^2)^{-n/2}\exp\left\{-\dfrac{1}{2\sigma^2}\sum_{i\in \mathcal{S}} (y_{i}-\psi(\bm{x}_{i};\bm{\beta}))^2\right\} & \text{(regression)}   \\  
%               & \\
% %              \prod_{i=1}^{m} \psi(\bm{x_i};\bm{\beta})^{y_i} \left(1-\psi(x_i;\bm{\beta})\right)^{1-y_i}, &  \text{classification}  
%              \prod\limits_{i\in \mathcal{S}} \dfrac{\exp\{\psi_{y_i}(\bm{x}_{i};\bm{\beta})\}}{\sum_{t=1}^K \exp\{\psi_{t}(\bm{x}_i;\bm{\beta})\}} &  \text{(classification),}  
%              \end{array}  
% \right.  
% \end{equation} 

%, note that we use $\pi(\mathcal{B}|\bm{\beta}, \sigma^2)$ instead of $\pi(\mathcal{B}|\bm{\beta})$ for unification without effecting the result.

%The EMVS approach is efficient in identifying potential sparse high posterior probability submodels on high-dimensional regression \citep{Rockova14} and classification problem \citep{McDermott16}. These characteristics are helpful for large neural network computation, thus we refer to the stochastic version of the EMVS algorithm as Expectation Stochastic-Maximization (ESM).

\subsection{Empirical Bayesian with approximate priors}

To speed up the inference, we propose the AEB method by sampling $\bm{\beta}$ and optimizing $\sigma^2, \delta, \bm{\gamma}$, where uncertainty of the hyperparameters is not considered. Because the binary variable $\bm{\gamma}$ is harder to optimize directly, we consider optimizing the adaptive posterior $\E_{\bm{\gamma}|\cdot, \mathcal{D}}\left[\pi(\bm{\beta}, \sigma^2, \delta, \bm{\gamma}|\mathcal{D})\right]$ \footnote{$\E_{\bm{\gamma}|\cdot, \mathcal{D}}[\cdot]$ is short for $\E_{\bm{\gamma}|\beta^{(k)}, \sigma^{(k)},\delta^{(k)}, \mathcal{D}}[\cdot]$.} instead. Due to the limited memory, which restricts us from sampling directly from $\mathcal{D}$, we choose to sample $\bm{\beta}$ from $\E_{\bm{\gamma}|\cdot,\mathcal{D}}\left[\E_{\mathcal{B}}\left[\pi(\bm{\beta}, \sigma^2, \delta, \bm{\gamma}|\mathcal{B})\right]\right]$ \footnote{ $\E_{\mathcal{B}}[\pi(\bm{\beta}, \sigma^2, \delta, \bm{\gamma}|\mathcal{B})]$ denotes $\int_{\mathcal{D}}\pi(\bm{\beta}, \sigma^2, \delta, \bm{\gamma}|\mathcal{B})d\mathcal{B}$}. By Fubini's theorem and Jensen's inequality, we have
\begin{equation}
\begin{split}
\label{js_}
&\log \E_{\bm{\gamma}|\cdot,\mathcal{D}}\left[\E_{\mathcal{B}}\left[\pi(\bm{\beta}, \sigma^2, \delta, \bm{\gamma}|\mathcal{B})\right]\right]=\log \E_{\mathcal{B}}\left[\E_{\bm{\gamma}|\cdot,\mathcal{D}}\left[\pi(\bm{\beta}, \sigma^2, \delta, \bm{\gamma}|\mathcal{B})\right]\right]\\
\geq &\E_{\mathcal{B}}\left[\log  \E_{\bm{\gamma}|\cdot,\mathcal{D}}\left[\pi(\bm{\beta}, \sigma^2, \delta, \bm{\gamma}|\mathcal{B})\right]\right]\geq \E_{\mathcal{B}}\left[\E_{\bm{\gamma}|\cdot,\mathcal{D}}\left[\log\pi(\bm{\beta}, \sigma^2, \delta, \bm{\gamma}|\mathcal{B})\right]\right].\\
\end{split}
\end{equation}
Instead of tackling $\pi(\bm{\beta}, \sigma^2, \delta, \bm{\gamma}|\mathcal{D})$ directly, we propose to iteratively update the lower bound $Q$
\begin{equation}
\begin{split}
\label{obj_expect}
&Q(\bm{\beta}, \sigma, \delta|\bm{\beta}^{(k)}, \sigma^{(k)}, \delta^{(k)})=\E_{\mathcal{B}}\left[\E_{\bm{\gamma}|\mathcal{D}}\left[\log\pi(\bm{\beta}, \sigma^2, \delta, \bm{\gamma}|\mathcal{B})\right]\right].
\end{split}
\end{equation}
Given $(\bm{\beta}^{(k)}, \sigma^{(k)},\delta^{(k)})$ at the k-th iteration, we first sample $\bm{\beta}^{(k+1)}$ from $Q$, then optimize $Q$ with respect to $\sigma, \delta$ and $\E_{\bm{\gamma}_l|\cdot,\mathcal{D}}$ via SA, where $\E_{\bm{\gamma}_l|\cdot,\mathcal{D}}$ is used since $\bm{\gamma}$ is treated as unobserved variable.
% A simplified approach was used for point estimation in a neural-graphical model in a structured prediction task and proved to be useful \citep{Zhang17}.
To make the computation easier, we decompose our $Q$ as follows:
\begin{equation}
\begin{split}
&Q(\bm{\beta},\sigma, \delta|\bm{\beta}^{(k)}, \sigma^{(k)}, \delta^{(k)}) = Q_1(\beta, \sigma|\bm{\beta}^{(k)}, \sigma^{(k)}, \delta^{(k)}) + Q_2(\delta|\bm{\beta}^{(k)}, \sigma^{(k)}, \delta^{(k)})+C,
\end{split}
\end{equation}
Denote $\mathcal{X}$ and  $\mathcal{C}$ as the sets of the indices of sparse and non-sparse layers, respectively. We have:
\begin{equation}
\begin{split}
\label{lab: q1}
&Q_1(\bm{\beta}|\bm{\beta}^{(k)}, \sigma^{(k)},\delta^{(k)})=\underbrace{\frac{N}{n}\log \pi(\mathcal{B}|\bm{\beta})}_\text{log likelihood}-\underbrace{\sum_{l\in \mathcal{C}}\sum_{j\in p_l} \dfrac{\beta_{lj}^2}{2\sigma_0^2}}_\text{non-sparse layers $\mathcal{C}$} -\dfrac{p+\nu+2}{2} \log(\sigma^2)\\
&-\sum_{l\in \mathcal{X}}\sum_{j\in p_l} \underbrace{[\dfrac{|\beta_{lj}|\overbrace{\E_{\bm{\gamma}_l|\cdot,\mathcal{D}} \left[\dfrac{1}{v_0(1-\gamma_{lj})}\right]}^{ \kappa_{lj0}}}{\sigma}+\dfrac{\beta_{lj}^2\overbrace{\E_{\bm{\gamma}_l|\cdot,\mathcal{D}} \left[\dfrac{1}{v_1 \gamma_{lj}}\right]}^{ \kappa_{lj1}}}{2\sigma^2}]}_\text{deep SSGL priors in sparse layers $\mathcal{X}$}- \dfrac{\nu\lambda}{2\sigma^2}
\end{split}
\end{equation}
\begin{equation}
\begin{split}
\label{eq:q2}
&Q_2(\delta_l|\bm{\beta}^{(k)}_l, \delta^{(k)}_l) =\sum_{l\in \mathcal{X}} \sum_{j\in p_l}\log\left(\dfrac{\delta_l}{1-\delta_l}\right) \overbrace{\E_{\gamma_l|\cdot,\mathcal{D}}\gamma_{lj}}^{ \rho_{lj}} + (a-1) \log(\delta_l) + (p_l+b-1) \log(1-\delta_l),
\end{split}
\end{equation}
where $\bm{ \rho}, \bm{ \kappa}$, $\sigma$ and $\delta$ are to be estimated in the next section.
% where $\kappa_{0lj}=\E_{\bm{\gamma}_l|\cdot} \left[\dfrac{1-\gamma_{lj}}{v_0}\right]$, $\kappa_{1lj}=\E_{\bm{\gamma}_l|\cdot} \left[\dfrac{\gamma_{lj}}{v_1}\right]$.

\subsection{Empirical Bayesian via stochastic approximation}

% In linear regression, EMVS \citep{Rockova14} can be used to obtain the point estimates of $\bm{\beta}$ and all the priors. However, when the log likelihood function is non-convex, the radical update of EMVS no longer yields robust hyperparameters. By contrast, a sampling algorithm is more efficient in escaping local traps. 
To simplify the notation, we denote the vector $(\bm{\rho}, \bm{\kappa}, \sigma, \delta)$ by $\bm{\theta}$. Our interest is to obtain the optimal $\bm{\theta}_*$ based on the asymptotically correct distribution $\pi(\bm{\beta}, \bm{\theta}_*)$. This implies that we need to obtain an estimate $\bm{\theta}_*$ that solves a fixed-point formulation $\int g_{\bm{\theta_*}}(\bm{\beta})\pi(\bm{\beta}, \bm{\theta_*})d\bm{\beta}=\bm{\theta}_*$, where $g_{\bm{\theta}}(\bm{\beta})$ is inspired by EMVS to obtain the optimal $\bm{\theta}$ based on the current $\bm{\beta}$. Define the random output $g_{\bm{\theta}}(\bm{\beta})-\bm{\theta}$ as $H(\bm{\beta}, \bm{\theta})$ and the mean field function $h(\bm{\theta}):=\E[H(\bm{\beta}, \bm{\theta})]$. The stochastic approximation algorithm can be used to solve the fixed-point iterations:

\begin{itemize}
\item[(1)] Sample $\bm{\beta}^{(k+1)}$ from a transition kernel $\Pi_{\bm{\theta}^{(k)}}(\bm{\beta})$, which yields the distribution $\pi(\bm{\beta}, \bm{\theta}^{(k)})$,
\item[(2)] Update $\bm{\theta}^{(k+1)}=\bm{\theta}^{(k)}+\omega^{(k+1)} H(\bm{\theta}^{(k)}, \bm{\beta}^{(k+1)})=\bm{\theta}^{(k)}+\omega^{(k+1)} (h(\bm{\theta}^{(k)})+\Omega^{(k)}).$
\end{itemize}
where $\omega^{(k+1)}$ is the step size. The equilibrium point $\bm{\theta}_*$ is obtained when the distribution of $\bm{\beta}$ converges to the invariant distribution $\pi(\bm{\beta}, \bm{\theta}_*)$. The stochastic approximation \citep{Albert90} differs from the Robbins–Monro algorithm in that sampling $\bm{\beta}$ from a transition kernel instead of a distribution introduces a Markov state-dependent noise $\Omega^{(k)}$ \citep{andrieu05}. In addition, since variational technique is only used to approximate the priors, and the exact likelihood doesn't change, the algorithm falls into a class of adaptive SG-MCMC instead of variational inference.

% where $\bm{\theta}$ can be viewed as a vector $(\bm{\rho}, \bm{\tau}, \sigma, \delta)$, the random output $H(\bm{\theta}^{(k)},\bm{\beta}^{(k+1)})$ is equal to 

% To obtain a tight lower bound with respect to the Jensen inequality in (\ref{js_}), we need to obtain an accurate estimation of $\E_{\bm{\gamma}|\cdot, \mathcal{D}}$ (denoted as $\bm{\rho}$). In addition, the lower bound needs to be strictly increasing to converge. However, neither of two conditions are satisfied due to the use of random mini-batch and the non-linear objective functions, which have no close-form update. Instead, we can only estimate $\bm{\gamma}$, $\sigma$ and $\delta$ based on $\mathcal{B}$, this suggests us to approximate $\E_{\gamma|\mathcal{D}}$, $\E[\sigma|\mathcal{D}]$ and $\E[\delta|\mathcal{D}]$ via SA (see details in Appendix A.2).

% The routine for SA to optimize expectation is as follows:

% is defined by is the unbiased estimate of the mean field mapping $h(\hat \theta):=\lim_{k\rightarrow \infty} \E_{\hat \theta}[H(\hat \theta, x_k)]$. When we can sample from $\mu_{{\hat \theta_k}}$ directly and $\mu_{{\hat \theta_k}}$ is in an exponential family, $H(\hat \theta, x)$ is equivalent to the unbiased estimate of the natural gradient of the KL divergence $KL(f_{ \bar \theta}(x)|| f_{\hat \theta_k}(x))$

Regarding the updates of $g_{\bm{\theta}}(\bm{\beta})$ with respect to $\bm{\rho}$, we denote the optimal $\bm{\rho}$ based on the current $\bm{\beta}$ and $\delta$ by $\bm{\tilde\rho}$. We have that $\tilde \rho_{lj}^{(k+1)}$, the probability of $\beta_{lj}$ being dominated by the $L_2$ penalty is
\begin{equation}
\label{eq:10}
\tilde \rho_{lj}^{(k+1)}=\E_{\bm{\gamma}_l|\cdot, \mathcal{B}} \gamma_{lj} = \mathrm{P}(\gamma_{lj}=1|\bm{\beta}_l^{(k)}, \delta_l^{(k)}) =\dfrac{a_{lj}}{a_{lj}+b_{lj}},
\end{equation}
where $a_{lj}=\pi(\beta_{lj}^{(k)}|\gamma_{lj}=1)\mathrm{P}(\gamma_{lj}=1|\delta^{(k)}_l)$ and $b_{lj}=\pi(\beta_{lj}^{(k)}|\gamma_{lj}=0)\mathrm{P}(\gamma_{lj}=0|\delta^{(k)}_l)$. The choice of Bernoulli prior enables us to use $\mathrm{P}(\gamma_{lj}=1|\delta^{(k)}_l)=\delta^{(k)}_l$.

% The physical meaning of $\E_{\bm{\gamma}|\cdot}\gamma_j$ in $Q_2$ is the probability $\rho_j$ that has a large effect on $\beta_j$ in the model, where $\bm{\rho}\in\mathbb{R}^p$.

% When $p$ is large, numerical problems occur in some extremely cases, e.g. $\sigma^2$ decreases sharply. To overcome that issue, we can rewrite
% \begin{equation}
% E_{\gamma|\cdot} \gamma_j =\dfrac{1}{1+b_j/a_j}=\dfrac{1}{1+\dfrac{(1-\delta_{(t)})v_1}{\delta_{(t)} v_0} \exp\{\dfrac{\beta_j^2}{2\sigma^2}(\dfrac{1}{v_1}-\dfrac{1}{v_0})\}}.
% \end{equation}

Similarly, as to $g_{\bm{\theta}}(\bm{\beta})$ w.r.t. $\bm{\kappa}$, the optimal $\tilde \kappa_{lj0}$ and $\tilde \kappa_{lj1}$ based on the current ${\rho_{lj}}$ are given by:
\begin{equation}
\begin{split}
\label{eq:11}
\tilde \kappa_{lj0}&=\E_{\bm{\gamma}_l|\cdot, \mathcal{B}} \left[\dfrac{1}{v_0 (1-\gamma_{lj})}\right]=\dfrac{1- \rho_{lj}}{v_0}; \text{ }\tilde \kappa_{lj1}=\E_{\bm{\gamma}_l|\cdot, \mathcal{B}} \left[\dfrac{1}{v_1 \gamma_{lj}}\right]=\dfrac{ \rho_{lj}}{v_1}.
\end{split}
\end{equation}

% The general strategy of recursion follows ${\beta} \leftarrow \bm{\beta} + \epsilon \nabla_{\bm{\beta}} {\pi}(\bm{\beta}, \sigma^2, \delta|\mathcal{B})$ where $\epsilon$ is the learning rate. 

% , approximation of Eq.(\ref{lab: q1}) is often required, e.g. normal approximation \citep{Gelman13}, VI \citep{Drugowitsch13} or EP \citep{Gelman13}. In deep learning framework, due to the success of

% In order to optimize $Q_1$ with respect to $\sigma$, by denoting $\diag\{\kappa_i\}_{i=1}^p$ as $\bm{\mathcal{V}}$, following the formulation in \citet{McDermott16} and \citet{Rockova14} we have:

% \begin{equation}  
% \label{eq:13}
% \sigma^{(k+1)}=\sqrt{\dfrac{\left\Vert\left(\bm{\mathcal{V}}^{1/2}\bm{\beta}^{(k+1)}\rm\right) \right\Vert^2+ \nu\lambda}{p+\nu+2}}
% \end{equation}

% \begin{equation}  
% \label{eq:13}
% \sigma^{(k+1)}=\left\{  
%              \begin{array}{lr}  
%              \sqrt{\dfrac{\frac{N}{n}\sum_{i\in\mathcal{S}} \left(y_i-\psi(\bm{x}_i;\bm{\beta}^{(k+1)})\right)^2 +||\bm{\mathcal{V}}^{1/2}\bm{\beta}^{(k+1)}||^2+ \nu\lambda}{N+p+\nu}} & \text{(regression)},   \\  
%              \sqrt{\dfrac{\left\Vert\left(\bm{\mathcal{V}}^{1/2}\bm{\beta}^{(k+1)}\rm\right) \right\Vert^2+ \nu\lambda}{p+\nu+2}} &  \text{(classification).}  
%              \end{array}  
% \right.  
% \end{equation}

To optimize $Q_1$ with respect to $\sigma$, by denoting $\diag\{\kappa_{0li}\}_{i=1}^{p_l}$ as $\bm{\mathcal{V}_{0l}}$, $\diag\{\kappa_{1li}\}_{i=1}^{p_l}$ as $\bm{\mathcal{V}_{1l}}$ we have:
\begin{equation}  
\label{eq:13}
\tilde \sigma^{(k+1)}=\left\{  
             \begin{array}{lr}  
             \dfrac{R_b+\sqrt{R_b^2+4R_a R_c}}{2R_a} & \text{(regression)},   \\  
             \dfrac{C_b+\sqrt{C_b^2+4C_a C_c}}{2C_a} &  \text{(classification),}  
             \end{array}  
\right.  
\end{equation}
% \begin{equation}  
% \label{eq:13}
% \sigma^{(k+1)}=\left\{  
%              \begin{array}{lr}  
%              \sqrt{\dfrac{I +J+ \nu\lambda}{N+\sum_{l\in\mathcal{X}}p_l+\nu}} & \text{(regression)},   \\  
%              \sqrt{\dfrac{J+ \nu\lambda}{\sum_{l\in\mathcal{X}}p_l+\nu+2}} &  \text{(classification),}  
%              \end{array}  
% \right.  
% \end{equation}
where $R_a=N+\sum_{l\in\mathcal{X}}p_l+\nu$, $C_a=\sum_{l\in\mathcal{X}}p_l+\nu+2$, $R_b=C_b=\sum_{l\in\mathcal{X}}||\bm{\mathcal{V}}_{0l}\bm{\beta}_l^{(k+1)}||_1$, $R_c=I+J+\nu\lambda$, $C_c=J+\nu\lambda$, $I=\frac{N}{n}\sum_{i\in\mathcal{S}} \left(y_i-\psi(\bm{x}_i;\bm{\beta}^{(k+1)})\right)^2$, $J=\sum_{l\in\mathcal{X}}||\bm{\mathcal{V}}_{1l}^{1/2}\bm{\beta}_l^{(k+1)}||^2$.\footnote[2]{The quadratic equation has only one unique positive root. $\|\cdot\|$ refers to $L_2$ norm, $\|\cdot\|_1$ represents $L_1$ norm.}

To optimize $Q_2$, a closed-form update can be derived from Eq.(\ref{eq:q2}) and Eq.(\ref{eq:10}) given batch data $\mathcal{B}$:
\begin{equation}
\label{eq:14}
\begin{aligned}
\tilde \delta_l^{(k+1)} &= \argmax_{\delta_l \in \mathbb{R}} Q_2(\delta_l|\bm{\beta}_l^{(k)}, \delta_l^{(k)})= \dfrac{\sum_{j=1}^{p_l} \rho_{lj} + a - 1}{a + b + p_l - 2}.
\end{aligned}
\end{equation}

% \subsection{Posterior Approximation and Optimization}
% The posterior average can be approximated through the weighted sample average $\E[\psi(\bm{{\beta}})]=\tfrac{\sum_{k=1}^{{k_{\max}}} \epsilon^{(k)} \psi(\bm{\beta}^{(k)})}{\sum_{k=1}^{{k_{\max}}}\epsilon^{(k)}}$ \citep{Welling11} to reduce the variance of the estimator. \citet{Teh16} showed a theoretical optimal learning rate $\epsilon^{(k)}\propto k^{-1/3}$ for SGLD to achieve faster convergence for posterior average, which is applied in Sec. 6.1.

% To speed up the computation and alleviate the entrapment of latent variables in poor local optima, we can tune $\tau$ to balance between local trap escape and stationary point convergence. This strategy is applied in Sec. 6.2-6.4.

\subsection{Pruning strategy}

There are quite a few methods for pruning neural networks including the oracle pruning and the easy-to-use magnitude-based pruning \citep{Molchanov2016PruningCN}. Although the magnitude-based unit pruning shows more computational savings \citep{Aidan18}, it doesn't demonstrate robustness under coarser pruning \citep{Song16, Aidan18}. Pruning based on the probability $\bm{\rho}$ is also popular in the Bayesian community, but achieving the target sparsity in sophisticated networks requires extra fine-tuning. We instead apply the magnitude-based weight-pruning to our Resnet compression experiments and refer to it as SGLD-SA, which is detailed in Algorithm \ref{alg: SGLD-EM}. The corresponding variant of SGHMC with SA is referred to as SGHMC-SA.

\section{Convergence Analysis}
The key to guaranteeing the convergence of the adaptive SGLD algorithm is to use Poisson's equation to analyze additive functionals. By decomposing the Markov state-dependent noise $\Omega$ into martingale difference sequences and perturbations, where the latter can be controlled by the regularity of the solution of Poisson's equation, we can guarantee the consistency of the latent variable estimators. 
\begin{theorem}[$L_2$ convergence rate]
\label{th:1}
For any $\alpha\in (0, 1]$, under assumptions in Appendix B.1, the algorithm satisfies: there exists a large enough constant $\lambda$ and an equilibrium $\bm{\theta}^*$ such that
\begin{equation*}
\begin{split}
\E\left[\|\bm{\theta}^{(k)}-\bm{\theta}^*\|^2\right]\leq \lambda k^{-\alpha}.
\end{split}
\end{equation*}
\end{theorem}
SGLD with adaptive latent variables forms a sequence of inhomogenous Markov chains and the weak convergence of $\bm{\beta}$ to the target posterior is equivalent to proving the weak convergence of SGLD with biased estimations of gradients. Inspired by \citet{Chen15}, we have:

\begin{corollary}
\label{cor:1}
Under assumptions in Appendix B.2, the random vector $\bm{\beta}^{(k)}$ from the adaptive transition kernel $\Pi_{\bm{\theta}^{(k-1)}}$ converges weakly to the invariant distribution $e^{\tau L(\bm{\beta}, \bm{\theta^*)}}$ as $\epsilon\rightarrow 0$ and $k\rightarrow \infty$.
\end{corollary}

% \begin{corollary}
% \label{cor:1}
% Given fixed stepsize $\epsilon$, the distribution of $\bm{\beta}^{(k)}$ converges weakly to the invariant distribution $e^{\tau L(\bm{\beta}, \bm{\theta}^*)}$ with rate $\alpha\in (0, 1]$ as $\epsilon\rightarrow 0$.
% \end{corollary}

The smooth optimization of the priors
makes the algorithm robust to bad initialization and avoids entrapment in poor local optima. In addition, the convergence to the asymptotically correct distribution enables us to combine simulated annealing \citep{Kirkpatrick83optimizationby}, simulated tempering \citep{ST}, parallel tempering \citep{PhysRevLett86} or (and) dynamical weighting \citep{wong97} to obtain better point estimates in non-convex optimization and more robust posterior averages in multi-modal sampling. %Because SGHMC proposed by \citet{Chen14} is more attractive in large scale problems, the convergence of SGLD-SA also applies to SGHMC-SA.

%If we update $\bm{\beta}$ k iterations, then update the latent variables and then update $\bm{\beta}$ $k$ iterations,..., we call it ESM-k.

% Extension to the $t$-prior is provided in the appendix, where the same framework can be applied.

% \subsubsection{Deterministic annealing to mitigate multimodality}

% To reduce the chances of getting stuck in a local mode, one might use the deterministic annealing variant of the EM algorithm (DAEM) \citep{Ueda98}, which aims to find a minimum of a tempered version of the free energy function, i.e. 

% $H^{(k)}(w, \delta)=\dfrac{1}{t} \log \sum_{\gamma} \pi(\beta, \delta, \gamma, \sigma|y)^t$ \ with $0<t \leq 1$.

% In the above formula, the temperature parameter $1/t$ controls the degree of separation between the multiple modes of $H^{(k)}$. Multiple modes begin to appear as the temperature decreases and $H^{(k)}$ gradually resembles the actual incomplete posterior. Thus the strategy to set a high temperature in the beginning and decrease it later keeps us from the poorly chosen starting values.

% To extend the ESM algorithm to the deterministic annealing ESM iterations, we simply need to adapt the E-step to a modified posterior distribution, which is to replace $p_i^*$ with 
% \begin{equation}
% p_i^t=\dfrac{a_i^t}{a_i^t+b_i^t}=\dfrac{1}{1+[\dfrac{(1-\delta_{(t)})v_1}{\delta_{(t)} v_0} \exp\{\dfrac{\beta^2}{2\sigma^2}(\dfrac{1}{v_1}-\dfrac{1}{v_0})\}]^t},
% \end{equation}
% where $a_j=\pi(\beta_j^{(k)}|\gamma_j=1)P(\gamma_j=1|\delta^{(k)})$ and $b_j=\pi(\beta_j^{(k)}|\gamma_j=0)P(\gamma_j=0|\delta^{(k)})$.

\begin{algorithm}[tb]
   \caption{SGLD-SA with SSGL priors}
   \label{alg: SGLD-EM}
\begin{algorithmic}
   \State {\textbf{Initialize:} $\bm{\beta}^{(1)}$, $\bm{\rho}^{(1)}$, $\bm{\kappa}^{(1)}$, $\sigma^{(1)}$ and $\delta^{(1)}$ from scratch, set target sparse rates $\mathbb{D}$, $\mathbb{f}$ and $\mathbb{S}$}
   \For {$k \leftarrow 1:{k_{\max}}$}
%   \State {Sample mini-batch $\mathcal{B}^{(k)}$}
   \State {\textbf{Sampling}}
%    \State{$\bm{\beta}^{(k+1)}\leftarrow\bm{\beta}^{(k)} + \epsilon^{(k)} \nabla_{\bm{\beta}} {\pi}(\bm{\beta}^{(k)}, \sigma^{(k)}, \delta^{(k)}|\mathcal{B}^{(k)})+\sqrt{2\epsilon^{(k)} T^{-1}}\bm{\eta}^{(k)},\ \where\ \bm{\eta}^{(k)} \sim \mathcal{N}(\bm{0}, \bm{I})$}
   \State {$\bm{\beta}^{(k+1)}\leftarrow\bm{\beta}^{(k)} + \epsilon^{(k)} \nabla_{\bm{\beta}}  Q(\cdot|\mathcal{B}^{(k)})+\mathcal{N}({0, 2\epsilon^{(k)} \tau^{-1}})$}
   
   \State {\textbf{Stochastic Approximation for Latent Variables}}
%   \State {Compute $\bm{\tilde \rho}^{(k+1)}$, $ \bm{\tilde \kappa}^{(k+1)}$, $\tilde \sigma^{(k+1)}$, and $\tilde \delta^{(k+1)}$ according to Eq.(\ref{eq:10}),Eq.(\ref{eq:11}),Eq.(\ref{eq:13}) and Eq.(\ref{eq:14})}.
    % \State {Update the variables following Eq.(\ref{eq:10}),Eq.(\ref{eq:11}),Eq.(\ref{eq:13}) and Eq.(\ref{eq:14})}.
   \State {\textbf{SA}: $\bm{\rho}^{(k+1)}\leftarrow (1-\omega^{(k+1)}) \bm{\rho}^{(k)}+\omega^{(k+1)} \bm{\tilde \rho}^{(k+1)}$ following Eq.(\ref{eq:10})}
   \State {\textbf{SA}: $\bm{\kappa}^{(k+1)}\leftarrow(1-\omega^{(k+1)}) \bm{\kappa}^{(k)}+\omega^{(k+1)} \bm{\tilde \kappa}^{(k+1)}$ following Eq.(\ref{eq:11})}
   \State {\textbf{SA}: ${\sigma}^{(k+1)}\leftarrow (1-\omega^{(k+1)}) {\sigma}^{(k)}+\omega^{(k+1)} {\tilde \sigma}^{(k+1)}$ following Eq.(\ref{eq:13})}
   \State {\textbf{SA}: $\delta^{(k+1)}\leftarrow(1-\omega^{(k+1)}) \delta^{(k)}+\omega^{(k+1)} \tilde \delta^{(k+1)}$ following Eq.(\ref{eq:14})}

   \If {\textbf{Pruning}}
   \State {Prune the bottom-$s\%$ lowest magnitude weights}
    \State {Increase the sparse rate $s\leftarrow \mathbb{S}(1-\mathbb{D}^{k/{\mathbb{f}}})$}
   \EndIf
   \EndFor
\end{algorithmic}
\vspace{0em}
\end{algorithm}

% \section{The Convergence Property}
% \input{convergence.tex}

% \section{Related Works}
% \input{relworks.tex}

\section{Experiments}
\subsection{Simulation of Large-p-Small-n Regression}
We conduct the linear regression experiments with a dataset containing $n=100$ observations and $p=1000$ predictors. $\mathcal{N}_p(0, \bm{\Sigma})$ is chosen to simulate the predictor values $\bm{X}$  (training set) where $\bm{\Sigma}=(\Sigma)_{i,j=1}^p$ with $\Sigma_{i,j}=0.6^{|i-j|}$. Response values $\bm{y}$ are generated from $\bm{X}\bm{\beta}+\bm{\eta}$, where $\bm{\beta}=(\beta_1,\beta_2,\beta_3,0,0,...,0)'$ and $\bm{\eta}\sim \mathcal{N}_n(\bm{0}, 3 \bm{I}_n)$. We assume $\beta_1\sim \mathcal{N}(3, \sigma_c^2)$, $\beta_2\sim \mathcal{N}(2, \sigma_c^2)$, $\beta_3\sim \mathcal{N}(1, \sigma_c^2)$,  $\sigma_c=0.2$. We introduce some hyperparameters, but most of them are uninformative. We fix $\tau=1, \lambda=1, \nu=1, v_1=10, \delta=0.5, b=p$ and set $a=1$.  The learning rate follows $\epsilon^{(k)}= 0.001\times k^{-\frac{1}{3}}$, and the step size is given by $\omega^{(k)}=10\times (k+1000)^{-0.7}$.  We vary $v_0$ and $\sigma$ to show the robustness of SGLD-SA to different initializations. In addition, to show the superiority of the adaptive update, we compare SGLD-SA with the intuitive implementation of the EMVS to SGLD and refer to this algorithm as SGLD-EM, which is equivalent to setting $\omega^{(k)}:=1$ in SGLD-SA. To obtain the stochastic gradient, we randomly select 50 observations and calculate the numerical gradient. SGLD is sampled from the same hierarchical model without updating the latent variables.

%Sensitivity analysis shows that three hyperparameters are important: $v_0$, $a$ and $\sigma$, which are used to identify and regularize the over-fitted space. 

We simulate $500,000$ samples from the posterior distribution, and also simulate a test set with 50 observations to evaluate the prediction.  As shown in Fig.\ref{fig:lr} (d), all three algorithms are fitted very well in the training set, however, SGLD fails completely in the test set (Fig.\ref{fig:lr} (e)), {indicating the over-fitting problem of SGLD without proper regularization when the latent variables are not updated}. Fig.\ref{fig:lr} (f) shows that although SGLD-EM successfully identifies the right variables, the estimations are lower biased. The reason is that SGLD-EM fails to regulate the right variables with $L_2$ penalty, and $L_1$ leads to a greater amount of shrinkage for $\bm{\beta}_1$, $\bm{\beta}_2$ and $\bm{\beta}_3$ (Fig. \ref{fig:lr} (a-c)), {implying the importance of the adaptive update via SA in the stochastic optimization of the latent variables}. In addition, from Fig.~\ref{fig:lr}(a), Fig. \ref{fig:lr}(b) and Fig.\ref{fig:lr}(c), we see that SGLD-SA is the only algorithm among the three that quantifies the uncertainties of $\beta_1$, $\beta_2$ and $\beta_3$ and always gives the best prediction as shown in Table.\ref{Linear_regression_UQ_test}. We notice that {SGLD-SA is fairly robust to various hyperparameters}.

For the simulation of SGLD-SA in logistic regression and the evaluation of SGLD-SA on UCI datasets, we leave the results in Appendix C and D.

\begin{table*}[t]  % Linear regression with UQ pars
\caption{Predictive errors in linear regression based on a test set considering different $v_0$ and $\sigma$}
\label{Linear_regression_UQ_test}
\begin{center}
\begin{small}
% \scriptsize
\begin{sc}
\begin{tabular}{lccccr}
\toprule
MAE / MSE & $v_0$=$0.01,\sigma$=2 & $v_0$=$0.1, \sigma$=2 & $v_0$=$0.01,\sigma$=1 &  $v_0$=$0.1, \sigma$=1 &\\
\midrule %% testing
SGLD-SA     & \textbf{1.89} / \textbf{5.56}  & \textbf{1.72} / \textbf{5.64}  & \textbf{1.48} / \textbf{3.51} & \textbf{1.54} / $\ $\textbf{4.42} &\\
SGLD-EM        & 3.49 / 19.31  & 2.23 / 8.22 & 2.23 / 19.28 & 2.07 / $\ $6.94 &\\
SGLD        & 15.85 / 416.39 & 15.85 / 416.39 & 11.86 / 229.38 & 7.72 / 88.90 &\\
\bottomrule
\end{tabular}
\end{sc}
\end{small}
\end{center}
\vspace{-1em}
\end{table*}

\begin{figure*}[!ht]
  \centering
  \subfigure[Posterior estimation of $\beta_{1}$.]{\includegraphics[scale=0.25]{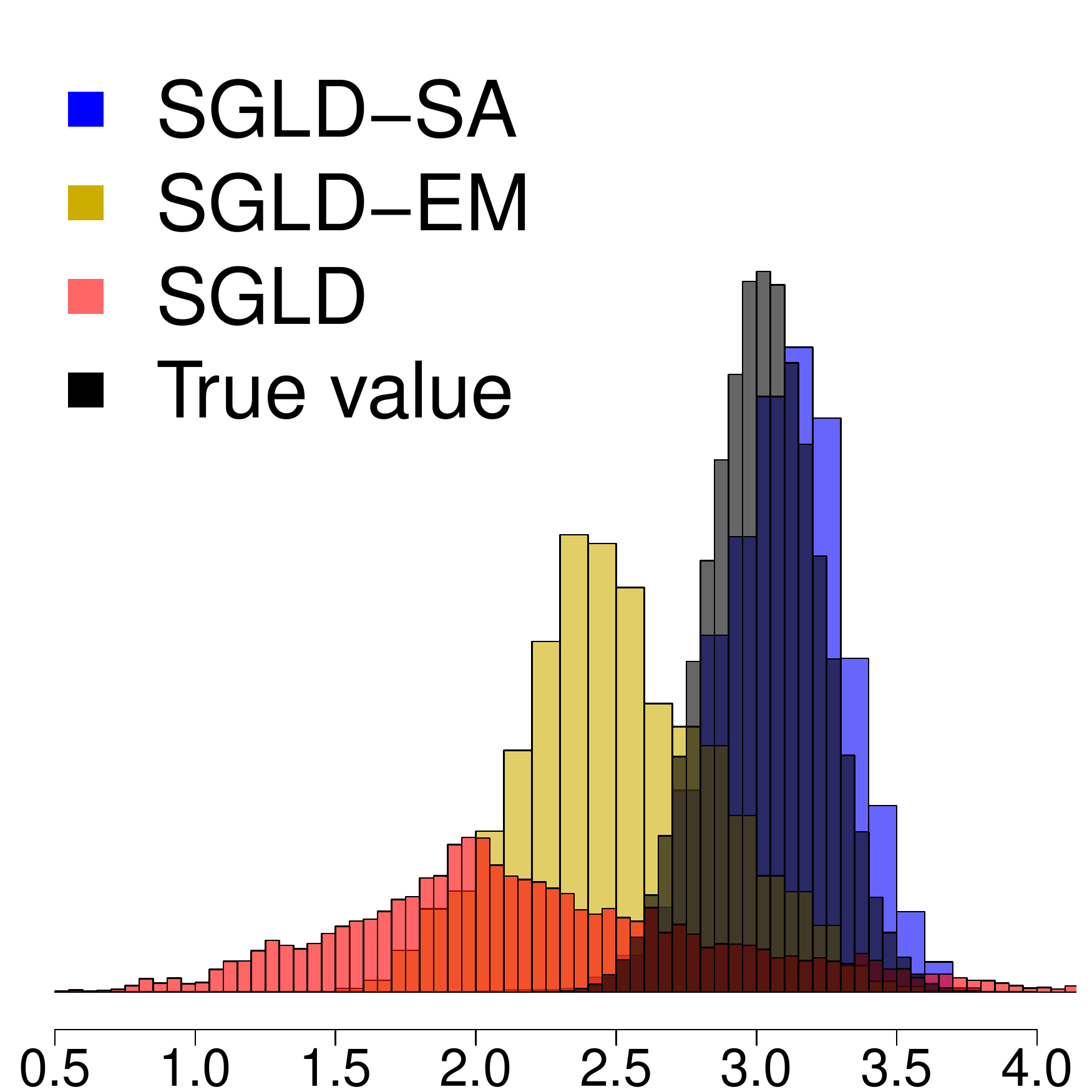}}\label{fig: 2a}
  \subfigure[Posterior estimation of $\beta_{2}$.]{\includegraphics[scale=0.25]{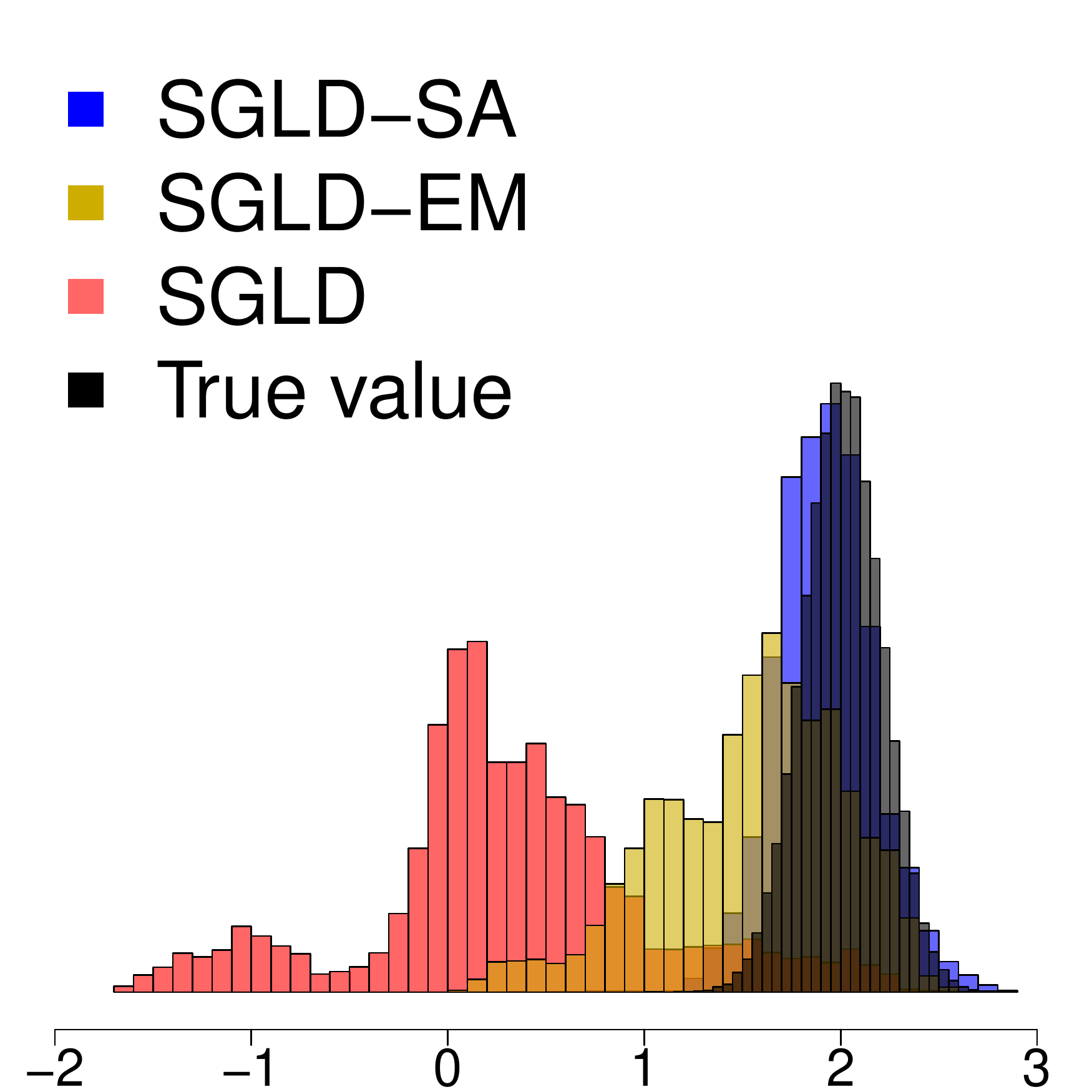}}\label{fig: 2b}
  \subfigure[Posterior estimation of $\beta_{3}$.]{\includegraphics[scale=0.25]{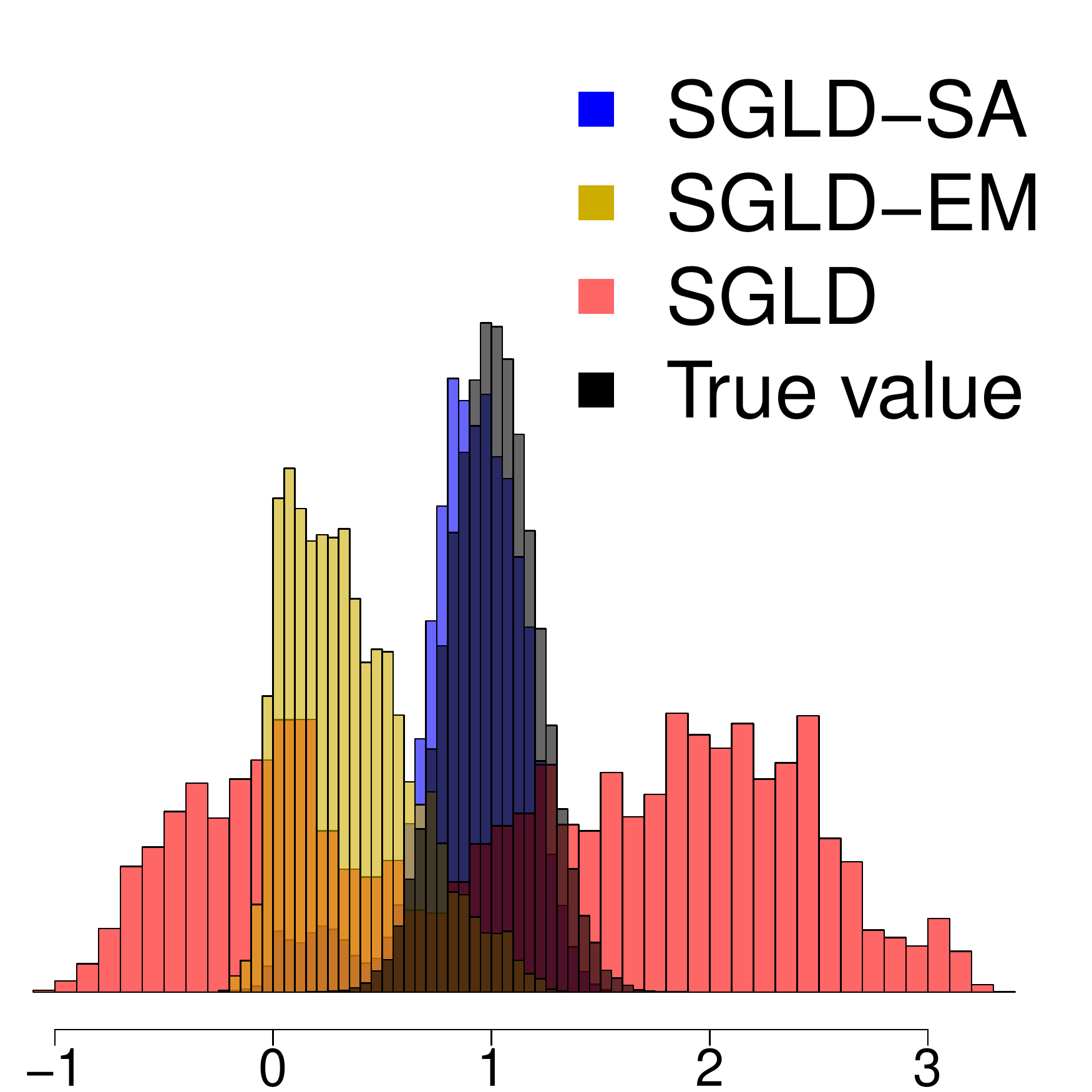}}\label{fig: 2c}
  \subfigure[Training performance.]{\includegraphics[scale=0.25]{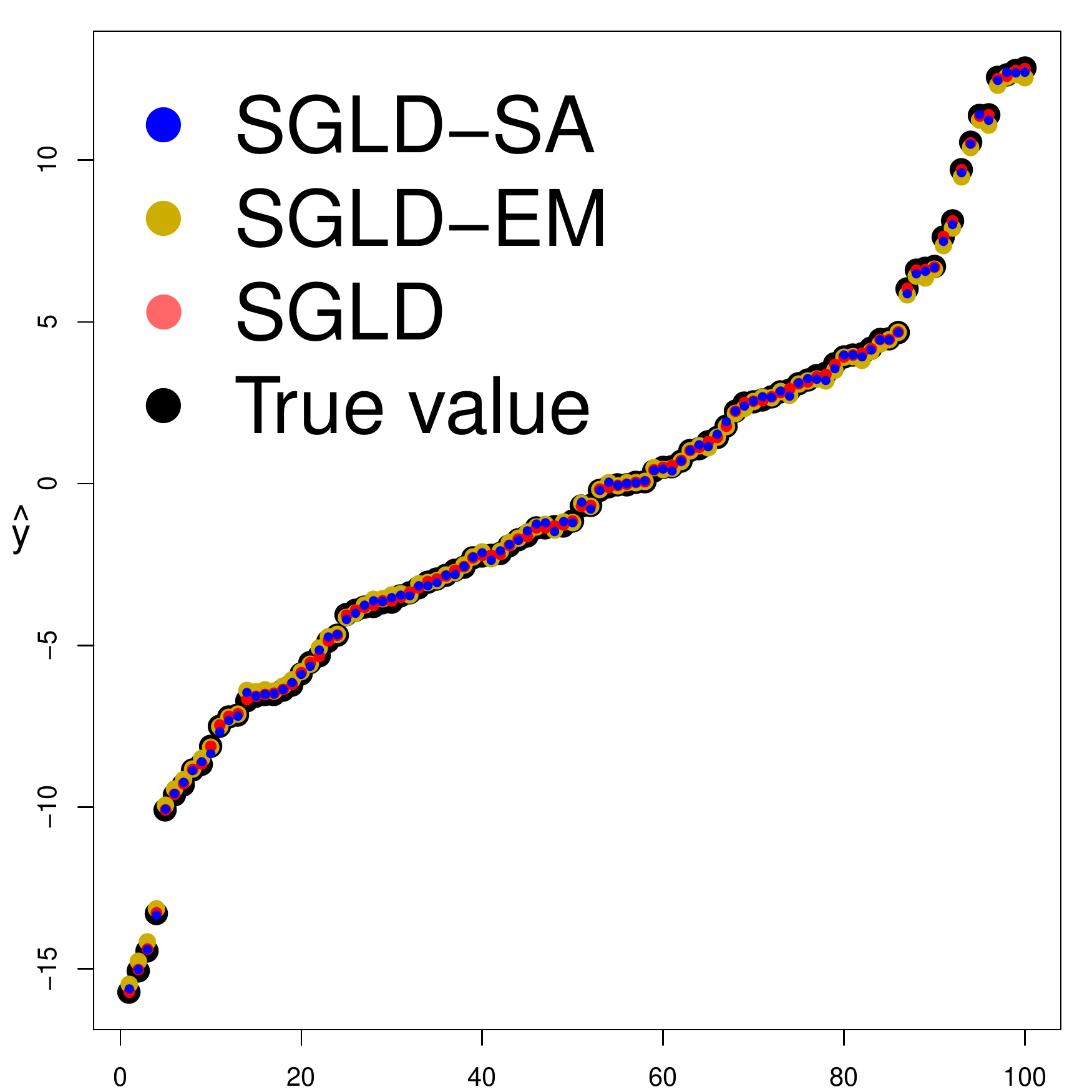}}\label{fig: 2d}
  \subfigure[Testing performance.]{\includegraphics[scale=0.25]{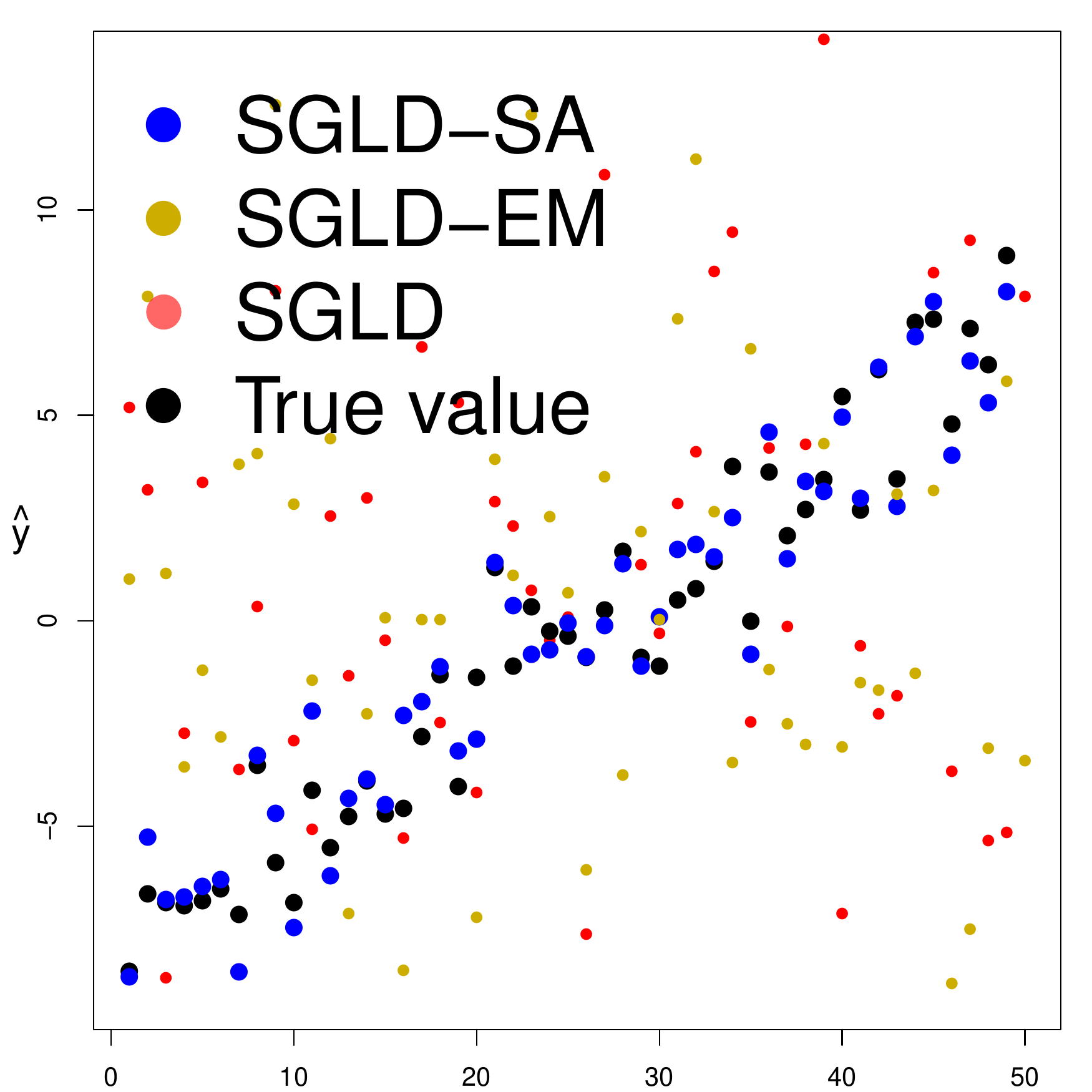}}\label{fig: 2e}
  \subfigure[Posterior mean vs truth.]{\includegraphics[scale=0.25]{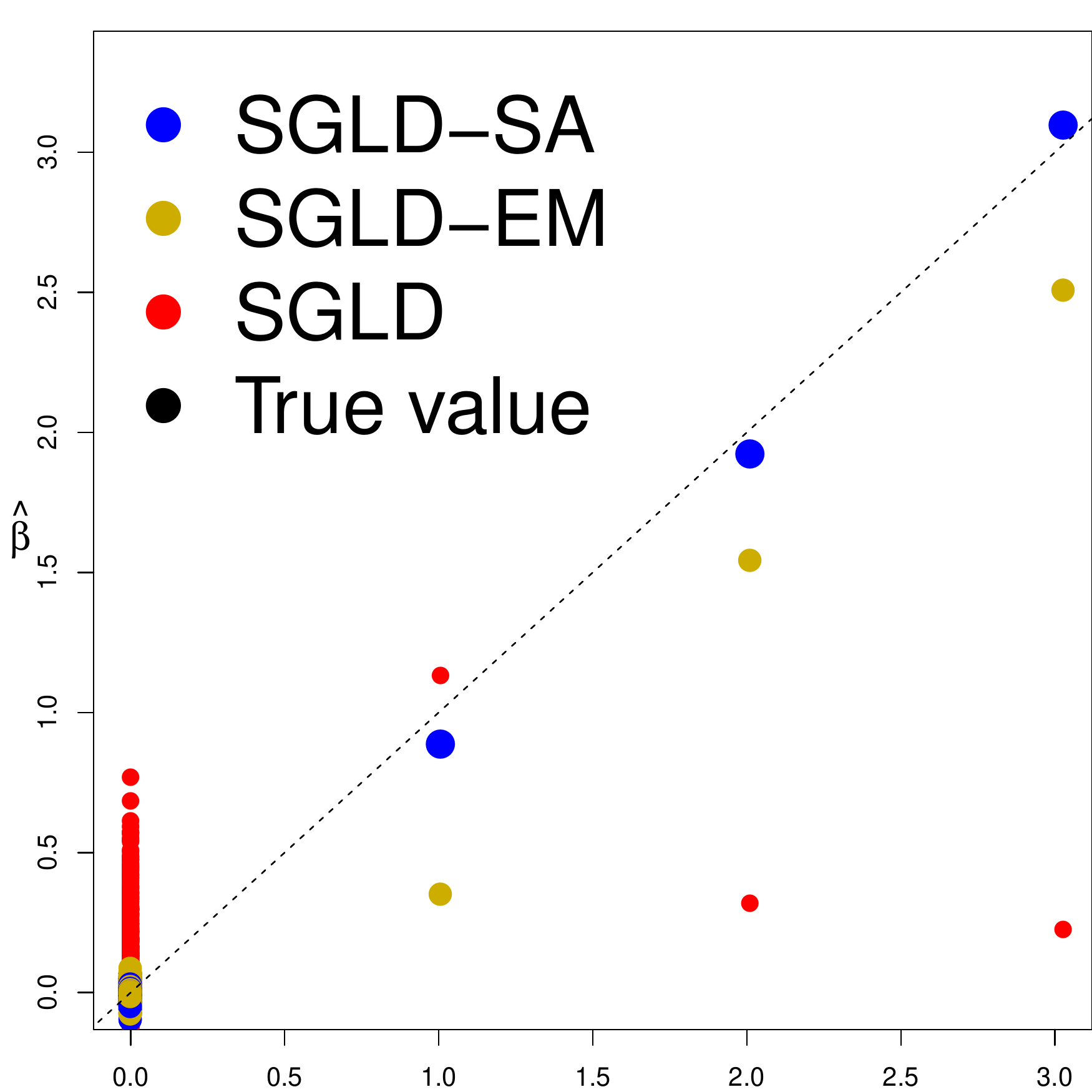}}\label{fig: 2f}
  \caption{Linear regression simulation when $v_0=0.1$ and $\sigma=1$.}
  \label{fig:lr}
  \vspace{-1em}
\end{figure*}

% Furthermore, as a byproduct of uncertainty test in Section 4.3, we also applied our framework to classify cats and dogs from Kaggle in Section 4.1. The first 20000 images of dog and cat images are used for training, the following 4000 for testing and the rest are used in Section 4.3. \textcolor{red}{OJO}The second DNN we used is the first mentioned CNN architecture in Keras Blog. We didn't tune our hyperparameters and just applied the same settings as in our first DNN. We also explored the possibility of adapting our framework to the pretrained VGG16 \citep{Simonyan16} network and do fine-tuning as suggested, a good performance can be achieved after 100 epochs to give accuracy $\sim$97\% within 3 hours using GPU 1080 Ti.

\subsection{Classification with Auto-tuning Hyperparameters}
\label{sec: mnist}

% It has 2 convolutional layers with a 2 $\times$ 2 max pooling after each layer and a 2-layer fully-connected layer \citep{Jarrett09}. The filter size in the convolutional layers is 5 $\times$ 5 and the feature maps are set to be 32 and 64, respectively. The fully-connected layers (FC) have 200 hidden nodes and 10 outputs. We use the rectified linear unit (ReLU) as the activation function between layers and employ a cross-entropy loss. 

The following experiments are based on non-pruning SG-MCMC-SA, the goal is to show that auto-tuning sparse priors are useful to avoid over-fitting. The posterior average is applied to each Bayesian model. We implement all the algorithms in Pytorch \citep{paszke2017}. The first DNN is a standard 2-Conv-2-FC CNN model of 670K parameters (see details in Appendix D.1).

%: It has two convolutional layers with a 2 $\times$ 2 max pooling after each layer and two fully-connected layers. The filter size in the convolutional layers is 5 $\times$ 5 and the feature maps are set to be 32 and 64, respectively  \citep{Jarrett09}. The fully-connected layers (FC) have 200 hidden nodes and 10 outputs. We use the rectified linear unit (ReLU) as the activation function between layers and employ a cross-entropy loss. 
The first set of experiments is to compare methods on the same model without using data augmentation (DA) and batch normalization (BN) \citep{Ioffe15}. We refer to the general CNN without dropout as {Vanilla}, with 50\% dropout rate applied to the hidden units next to FC1 as {Dropout}. %(DropConnect is not tested due to its similarity to Dropout). 

Vanilla and Dropout models are trained with Adam \citep{Kingma14} and Pytorch default parameters (with learning rate 0.001). We use SGHMC as a benchmark method as it is also sampling-based and has a close relationship with the popular momentum based optimization approaches in DNNs. {SGHMC-SA} differs from {SGHMC} in that {SGHMC-SA} keeps updating SSGL priors for the first FC layer while they are fixed in {SGHMC}. We set the training batch size $n=1000$, $a,b=p$ and $\nu,\lambda=1000$. The hyperparameters for SGHMC-SA are set to $v_0=1, v_1=0.1$ and $\sigma=1$ to regularize the over-fitted space. The learning rate is set to $5\times 10^{-7}$, and the step size is $\omega^{(k)}= 1\times (k+1000)^{-\frac{3}{4}}$. We use a thinning factor 500 to avoid a cumbersome system. Fixed temperature can also be powerful in escaping ``shallow" local traps \citep{Yuchen17}, our temperatures are set to $\tau=1000$ for MNIST and $\tau=2500$ for FMNIST.

%  The total sample size $N$ in Eq.(\ref{factor}) is set to $600,000$ instead of $60,000$.\footnote{Due to the use of DA in some experiments, we unify this number as $N=600,000$} 

% The results are averaged over 5 repeated experiments for a comprehensive evaluation. 
The four CNN models are tested on MNIST and Fashion MNIST (FMNIST) \citep{Xiao17} dataset. Performance of these models is shown in Tab.\ref{tab:classi}. Compared with SGHMC, our SGHMC-SA outperforms SGHMC on both datasets. We notice the posterior averages from SGHMC-SA and SGHMC obtain much better performance than Vanilla and Dropout. {Without using either DA or BN, SGHMC-SA achieves {99.59}\% which outperforms some state-of-the-art models}, such as Maxout Network (99.55\%) \citep{Goodfellow13} and pSGLD (99.55\%) \citep{Li16} . In F-MNIST, SGHMC-SA obtains 93.01\% accuracy, outperforming all other competing models. %such as stochastic pooling (99.53\%) \citep{Zeiler13}, Maxout Network (99.55\%) \citep{Goodfellow13} and pSGLD (99.55\%) \citep{Li16} . In F-MNIST, SGHMC-SA obtains 93.01\% accuracy, outperforming all other competing models. 

To further test the performance, we apply DA %\footnote[1]{For FMNIST and CIFAR10, random erasing {\citep{Zhong17}} is also used in DA} 
and BN to the following experiments (see details in Appendix D.2) and refer to the datasets as $\text{DA-MNIST}$ and $\text{DA-FMNIST}$. All the experiments are conducted using a 2-Conv-BN-3-FC CNN of 490K parameters. %:  It has two convolutional layers with a $2\times2$ max pooling after each layer and three fully-connected layers with batch normalization applied to the first FC layer. The filter size in the convolutional layers is $4\times4$ and the feature maps are both set to 64. We use $256\times64\times10$ fully-connected layers. 
Using this model, we obtain the {state-of-the-art 99.75\%} on $\text{DA-MNIST}$ (200 epochs) and {94.38}\% on $\text{DA-FMNIST}$ (1000 epochs) as shown in Tab. \ref{tab:classi}. The results are noticeable, because posterior average is only conducted on a single shallow CNN. 

\begin{table}[t]
\caption{Classification accuracy using shallow networks}
\label{data-table}
\begin{center}
% \scriptsize
\small
\begin{sc}
\label{tab:classi}
\begin{tabularx}{0.7\textwidth}{lcccc}
\toprule
Dataset & $\text{MNIST}$ & $\text{DA-MNIST}$ &  $\text{FMNIST}$ & $\text{DA-FMNIST}$\\
\midrule
Vanilla     & 99.31 & 99.54  & 92.73 & 93.14  \\
Dropout     & 99.38 & 99.56  & 92.81 & 93.35  \\
\midrule
SGHMC       & 99.47 & 99.63  & 92.88 & 94.29   \\
\textbf{SGHMC-SA}    & \textbf{99.59} & \textbf{99.75} & \textbf{93.01} & \textbf{94.38}  \\
\bottomrule
\end{tabularx}
\end{sc}
\end{center}
\vspace{-2em}
\end{table}

% \begin{table}[t]
% \caption{Classification accuracy on MNIST using shallow networks}
% \label{data-table}
% \begin{center}
% \begin{small}
% \begin{sc}
% \label{tab:classi}
% \begin{tabular}{lccc}
% \toprule
% Dataset & $\text{MNIST}$ & $\text{DA-MNIST}$ & $\lowercase{a}\text{MNIST-5}$ \\
% \midrule
% Vanilla     & 99.31 & 99.54 & 99.75  \\
% Dropout     & 99.38 & 99.56 & 99.74   \\
% \midrule
% SGHMC       & 99.55 & 99.71 & 99.77   \\
% \textbf{SGHMC-SA}    & \textbf{99.60} & \textbf{99.75} & \textbf{99.79}  \\
% \bottomrule
% \end{tabular}
% \end{sc}
% \end{small}
% \end{center}
% \vspace{-1em}
% \end{table}

% \begin{table}[t]
% \caption{Classification accuracy on FMNIST}
% \label{data-table2}
% \begin{center}
% \begin{small}
% \begin{sc}
% \label{tab:classi2}
% \begin{tabular}{lccc}
% \toprule
% Dataset & $\text{FMNIST}$ & $\text{DA-FMNIST}$ & $\lowercase{a}\text{FMNIST-5}$ \\
% \midrule
% Vanilla             & 92.73 & 93.14  &  94.48 \\
% Dropout             & 92.81 & 93.35  &  94.53 \\
% \midrule
% SGHMC               & 92.93 & 94.29  & 94.64  \\
% \textbf{SGHMC-SA}   & \textbf{93.01} & \textbf{94.38} & \textbf{94.78} \\
% \bottomrule
% \end{tabular}
% \end{sc}
% \end{small}
% \end{center}
% \vspace{-1em}
% \end{table}

\subsection{Defenses against Adversarial Attacks}
\label{adversarial}
Continuing with the setup in Sec. \ref{sec: mnist}, the third set of experiments focuses on evaluating model robustness. We apply the \textit{Fast Gradient Sign} method \citep{Goodfellow14a} to generate the adversarial examples with one single gradient step: 
\begin{equation}
\bm{x}_{adv} \leftarrow \bm{x}-\zeta \cdot \sign\{\delta_{\bm{x}} \max_{\y} \log \mathrm{P}(\y|\bm{x})\}, \nonumber
\end{equation}
where $\zeta$ ranges from $0.1, 0.2, \ldots, 0.5$ to control the different levels of  adversarial attacks.
% We expect less robust models perform considerably well on a certain dataset due to over-tuning; however, as the degree of adversarial attacks increases, the performance decreases sharply. In contrast, more robust models should be less affected by these adversarial attacks.

Similar to the setup in \citet{Li17}, we normalize the adversarial images by clipping to the range $[0, 1]$. In Fig. \ref{fig: adversarial}(b) and Fig.\ref{fig: adversarial}(d), we see no significant difference among all the four models in the early phase. As the degree of adversarial attacks arises, the images become vaguer as shown in Fig.\ref{fig: adversarial}(a) and Fig.\ref{fig: adversarial}(c). The performance of {Vanilla} decreases rapidly, reflecting its poor defense against adversarial attacks, while {Dropout} performs better than {Vanilla}. But {Dropout} is still significantly worse than the sampling based methods. The advantage of {SGHMC-SA} over {SGHMC} becomes more significant when $\zeta>0.25$. In the case of $\zeta=0.5$ in MNIST where the images are hardly recognizable, both {Vanilla} and {Dropout} models fail to identify the right images and their predictions are as worse as random guesses. However, {SGHMC-SA} model achieves roughly {11\%} higher than these two models and {1\%} higher than {SGHMC}, which demonstrates the robustness of  {SGHMC-SA}.

% Regarding the running time used, 100 epochs of the traditional training (Vanilla and Dropout) costs us 12.3 minutes, 100 epochs of \textbf{EP} training takes 32.3 minutes, as a contrast, 100 epochs of \textbf{sSGHMC} (90 epochs of ESM point estimation and 10 epochs of SGHMC) only spend 13.5 minutes using GPU Tesla P100, which showed our speed advantage over $\alpha$-divergence minimization.

% \begin{figure}
% \begin{subfigure}{0.17\linewidth}
% \includegraphics[scale=0.34]{figures/MNIST.PNG}
% \caption{$\zeta=..., 0.3, ...$.}
% \label{fig: 3a}
% \end{subfigure}
% \hfill
% \begin{subfigure}{0.3\linewidth}
% \includegraphics[scale=0.42]{figures/adversarial_comparison.pdf}
% \caption{MNIST}
% \label{fig: 3b}
% \end{subfigure}
% \hspace{0.5em}
% \begin{subfigure}{0.17\linewidth}
% \includegraphics[scale=0.34]{figures/fmnist.PNG}
% \caption{$\zeta=..., 0.3, ...$.}
% \label{fig: 3c}
% \end{subfigure}
% \hfill
% \begin{subfigure}{0.3\linewidth}
% \includegraphics[scale=0.42]{figures/adversarial_comparison_fmnist.pdf}
% \caption{F-MNIST}
% \label{fig: 3d}
% \end{subfigure}
% \caption{Adversarial test accuracies based on adversarial images of different levels}
% \vspace{-1.5em}
% \end{figure}

\begin{figure*}[!ht]
  \centering
  \subfigure[$\zeta=...$.]{\includegraphics[scale=0.3]{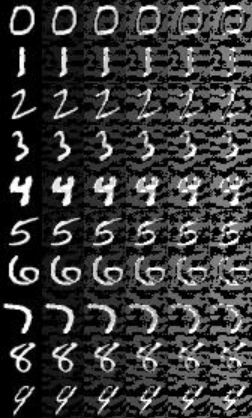}}\label{fig: 3a}
  \subfigure[MNIST]{\includegraphics[scale=0.35]{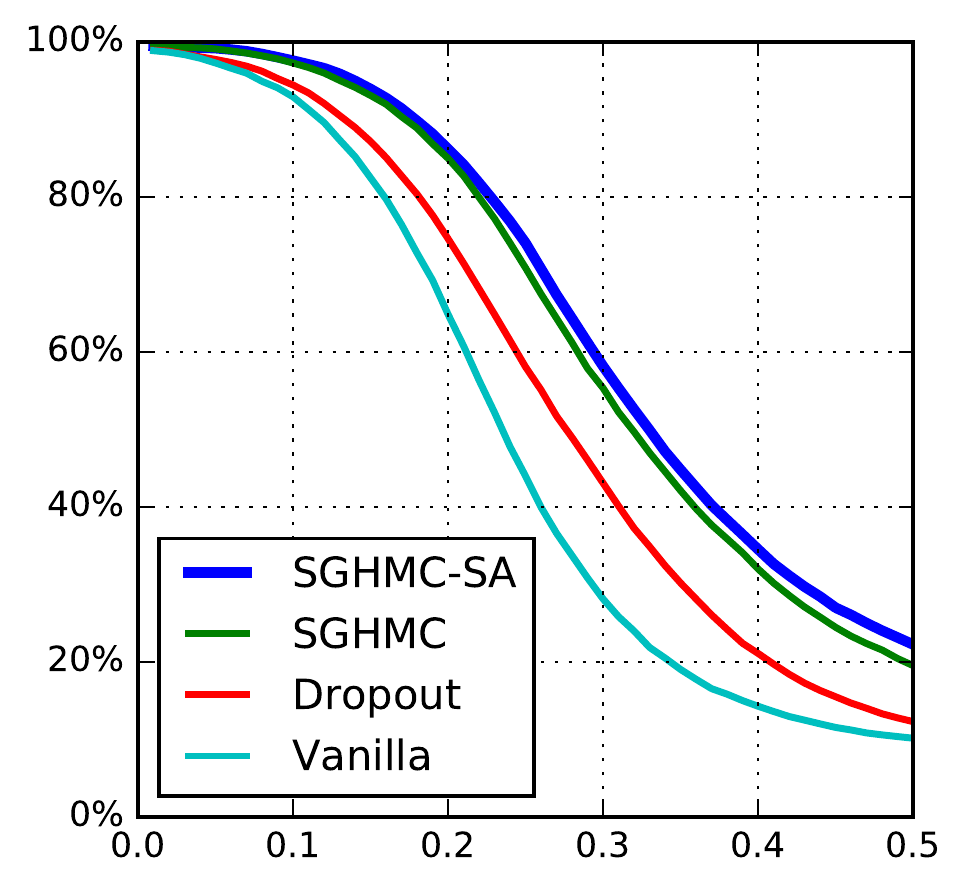}}\label{fig: 3b}\qquad
  \subfigure[$\zeta=...$.]{\includegraphics[scale=0.3]{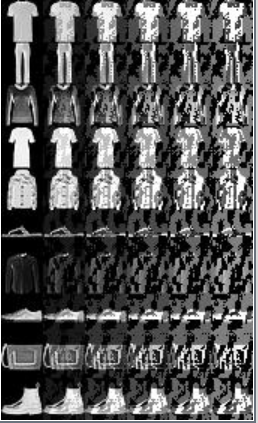}}\label{fig: 3c}
  \subfigure[FMNIST]{\includegraphics[scale=0.35]{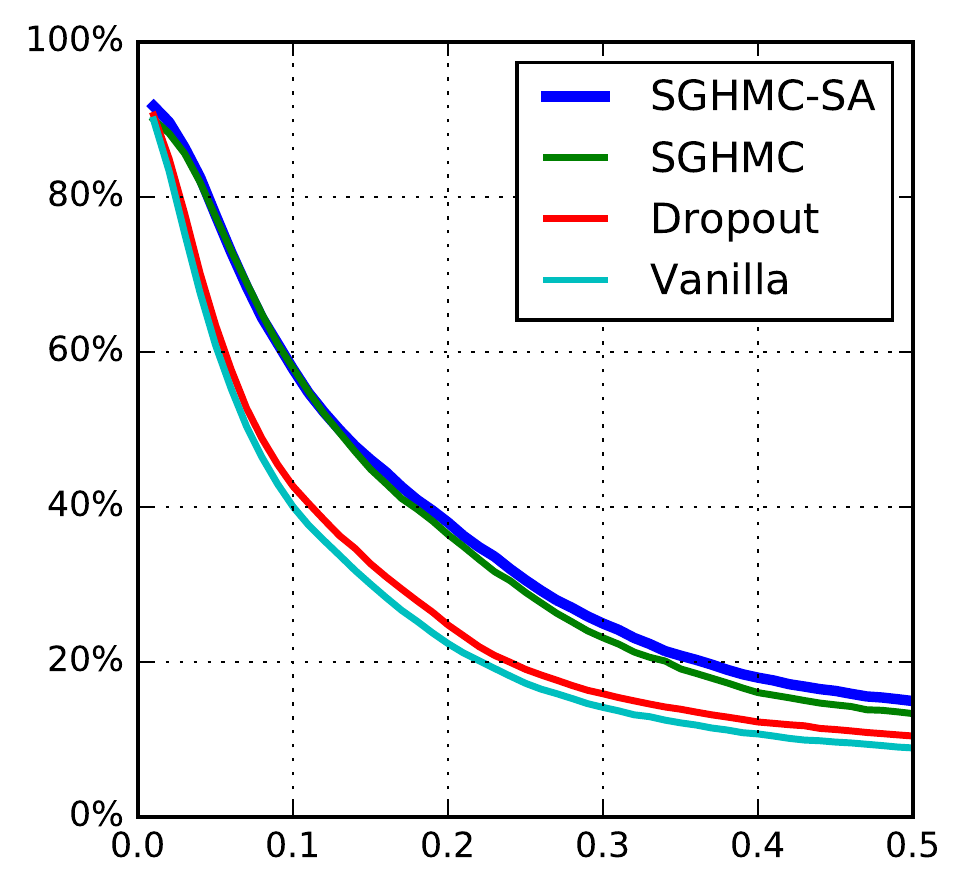}}\label{fig: 3d}
  \caption{Adversarial test accuracies based on adversarial images of different levels}
  \label{fig: adversarial}
  \vspace{-1.5em}
\end{figure*}

\subsection{Residual Network Compression}

Our compression experiments are conducted on the CIFAR-10 dataset \citep{Alex09} with DA. SGHMC and the non-adaptive SGHMC-EM are chosen as baselines. Simulated annealing is used to enhance the non-convex optimization and the methods with simulated annealing are referred to as A-SGHMC, A-SGHMC-EM and A-SGHMC-SA, respectively. We report the best point estimate.
% \footnote{Since we are using the factor $N/n$ from Eq.(\ref{factor}), 0.002 comes from $0.1\times \frac{50,000}{1,000}$}

We first use SGHMC to train a Resnet20 model and apply the magnitude-based criterion to prune weights to all convolutional layers (except the very first one). All the following methods are evaluated based on the same setup except for different step sizes to learn the latent variables. 
% The step size is $\omega^{(k)}=0.02\times (k+1000)^{-0.75}$ for A-SGHMC-SA, $\omega^{(k)}\equiv 1$ for A-SGHMC-EM and $\omega^{(k)}\equiv 0$ for A-SGHMC. 
The sparse training takes 1000 epochs. The mini-batch size is 1000. The learning rate starts from 2e-9 \footnote[2]{It is equivalent to setting the learning rate to 1e-4 when we don't multiply the likelihood with $\frac{N}{n}$.} and is divided by 10 at the 700th and 900th epoch. We set the inverse temperature $\tau$ to 1000 and multiply $\tau$ by 1.005 every epoch . We fix $\nu=1000$ and $\lambda=1000$ for the inverse gamma prior. $v_0$ and $v_1$ are tuned based on different sparsity to maximize the performance. The smooth increase of the sparse rate follows the pruning rule in Algorithm \ref{alg: SGLD-EM}, and $\mathbb{D}$ and $\mathbb{f}$ are set to $0.99$ and $50$, respectively. The increase in the sparse rate $s$ is faster in the beginning and slower in the later phase to avoid destroying the network structure. Weight decay in the non-sparse layers $\mathcal{C}$ is set as 25. %To demonstrate the advantage of the self-adapting penalty functions over the others, we compare a list of fixed $L_1$ penalties and $L_2$ penalties to the sparse layers $\mathcal{X}$ to show the strength of our proposed method.
%mitigate the potential problem of local traps\footnote{To accommodate for the standard settings in training Resnet with crossentropy loss, we ignore the factor $N/n$ from Eq.(\ref{factor}), it is equivalent to set $\tau=2\times10^4$ in Eq.(\ref{factor})}.
\begin{table}[t]
\caption{Resnet20 Compression on CIFAR10. When $\mathbb{S}=0.9$, we fix $v_0=0.005$, $v_1=$1e-5; When $\mathbb{S}=0.7$, we fix $v_0=0.1$, $v_1=$5e-5; When $\mathbb{S}=0.5$, we fix $v_0=0.1$, $v_1=$5e-4; When $\mathbb{S}=0.3$, we fix $v_0=0.5$, $v_1=$1e-3.}
\label{compression_1}
\begin{center}
\begin{small}
% \scriptsize
\begin{sc}
% \begin{tabular}{lccccc}
\begin{tabularx}{0.55\textwidth}{ccccc}
\toprule
 Methods $\backslash\ $ $\ \mathbb{S}$ & $30\%$ & $50\%$ & $70\%$ & $90\%$ \\
\midrule
% $L_1=1\times10^{-4}$   &  \textbf{94.09}  & \textbf{93.74}   & 92.98   & 89.47 \\
% % $L_1=3\times10^{-4}$   &  94.03  & 93.68   & \textbf{93.01}   & 89.62 \\
% % $L_1=1\times10^{-3}$   &  92.88   &  92.75  &  92.62  & 89.95 \\
% $L_1=3\times10^{-3}$   &  92.48  &  92.49  &  92.10  & \textbf{91.12} \\
% % $L_1=1\times10^{-2}$   &  89.50  &  89.79  &  90.07  & 89.83 \\
% % $L_1=3\times10^{-2}$   &  83.02  &  82.36  &  82.88  & 83.50 \\
% \midrule
% % $L_2=1\times10^{-4}$   & 94.15  &  93.78  &  93.12  & 89.37 \\
% % $L_2=1\times10^{-3}$   & 94.03  &  93.64  &  93.06  & 89.29 \\
% % $L_2=1\times10^{-2}$   & 94.17  & 93.82   &  93.17  & 90.11 \\
% $L_2=3\times10^{-2}$   & \textbf{94.20}  & \textbf{93.99}   &  93.37  & 91.15 \\
% $L_2=1\times10^{-1}$   &  94.02  &  93.96  &  \textbf{93.50}  & 91.20 \\
% $L_2=3\times10^{-1}$   &  93.72  &  93.57  &  93.37  & \textbf{91.27} \\
% $L_2=6\times10^{-1}$   &  93.30  &  93.09  &  93.00  & 91.40 \\
% $L_2=1\times10^{0}$   &  92.94  & 92.90   &  92.56  & 91.16 \\
% \midrule
% \textbf{SGHMC-SSG}    & \textbf{94.27}  &  93.94  &  \textbf{93.63}  &  \textbf{91.51} \\
A-SGHMC  &  94.07  & 94.16   &  93.16  & 90.59 \\
A-SGHMC-EM   &  94.18  & 94.19   &  93.41  & 91.26 \\
SGHMC-SA    & {94.13}  &  {94.11}  &  {93.52}  &  {91.45} \\
\textbf{A-SGHMC-SA}    & \textbf{94.23}  &  \textbf{94.27}  &  \textbf{93.74}  &  \textbf{91.68} \\
\bottomrule
\end{tabularx}
\end{sc}
\end{small}
\end{center}
\vspace{-1em}
\end{table}

As shown in Table \ref{compression_1}, A-SGHMC-SA doesn't distinguish itself from A-SGHMC-EM and A-SGHMC when the sparse rate $\mathbb{S}$ is small, but outperforms the baselines given a large sparse rate.
%all the other methods based on fixed $L_1$ and $L_2$ penalties.
The pretrained model has accuracy 93.90\%, however,the prediction performance can be improved to the state-of-the-art $94.27\%$ with 50\% sparsity. Most notably, we obtain 91.68\% accuracy based on 27K parameters (90\% sparsity) in Resnet20. By contrast, targeted dropout obtained 91.48\% accuracy based on 47K parameters (90\% sparsity) of Resnet32 \citep{Aidan18}, BC-GHS achieves 91.0\% accuracy based on 8M parameters (94.5\% sparsity) of VGG models \citep{Louizos17}. We also notice that when simulated annealing is not used as in SGHMC-SA, the performance will decrease by 0.2\% to 0.3\%. When we use batch size 2000 and inverse temperature schedule $\tau^{(k)}=20\times 1.01^k$, A-SGHMC-SA still achieves roughly the same level, but the prediction of SGHMC-SA can be 1\% lower than A-SGHMC-SA. 

%As the sparse rate increases, a larger regularization is needed to yield a satisfying performance.

\section{Conclusion}
% about running time, from Chen 2014, when n is large, restrict all the matrices to be diagonal, resulting in the same time complexity with SGLD
We propose a novel AEB method to adaptively sample from hierarchical Bayesian DNNs and optimize the spike-and-slab priors, which yields a class of scalable adaptive sampling algorithms in DNNs. We prove the convergence of this approach to the asymptotically correct distribution. By adaptively searching and penalizing the over-fitted parameters, the proposed method achieves higher prediction accuracy over the traditional SG-MCMC methods in both simulated examples and real applications and shows more robustness towards adversarial attacks. Together with the magnitude-based weight pruning strategy and simulated annealing, the AEB-based method,  A-SGHMC-SA, obtains the state-of-the-art performance in model compression. %For larger networks, longer training time might be needed. Existing techniques, such as gradient noise control and temperature tuning, for accelerating SG-MCMC simulations should also be helpful to this proposed method.

%\footnote{Code will be available upon request.} %Simulated examples and real applications show the robustness of our proposed method over the traditional SG-MCMC methods and the resistance to over-fitting.

% Interesting future directions include applying SG-MCMC-SA to active learning to learn from datasets of smaller size and proving posterior consistency and the consistency of variable selection under various shrinkage priors concretely.

% \section{Acknowledgments}

$\newline$
\textbf{Acknowledgments}

We would like to thank Prof. Vinayak Rao, Dr. Yunfan Li and the reviewers for their insightful comments. We acknowledge the support from the National Science Foundation (DMS-1555072, DMS-1736364, DMS-1821233 and DMS-1818674) and the GPU grant program from NVIDIA.

% \input{appendix.tex}


\bibliography{mybib}
\bibliographystyle{plainnat}
\newpage
$\newline$
\appendix
\begin{Large}
\begin{center}
    \textbf{Supplimentary Material for \textit{An Adaptive Empirical  Bayesian Method for Sparse Deep Learning}}
\end{center}
\end{Large}

\begin{center}
    Wei Deng, Xiao Zhang, Faming Liang, Guang Lin \\
    Purdue University, West Lafayette, IN, USA \\
    \{deng106, zhang923, fmliang, guanglin\}@purdue.edu
\end{center}

In this supplementary material, we review the related methodologies in $\S$\ref{review}, prove the convergence in $\S$\ref{converge}, present additional simulation of logistic regression in $\S$\ref{logistic_reg}, illustrate more regression examples on UCI datasets in $\S$\ref{reg_UCI}, and show the experimental setup in $\S$\ref{exp_setup}.

\section{Stochastic Approximation}
\label{review}
% \subsection{Fokker-Planck Equation}

% The Fokker-Planck equation (FPE) can be formulated from the time evolution of the conditional distribution for a stationary random process. Denoting the probability density function of the random process at time t by $q(t, \beta)$, where $\beta$ is the parameter, the stochastic dynamics is given by 
% \begin{equation}
% \nabla_{t} q(t, \beta)=\nabla_{\bm{\beta}}(\nabla_{\bm{\beta}}(-L(\beta)) q(t,\beta))+\nabla_{\bm{\beta}}^2 q(t,\beta).
% \end{equation}
% It is known that as $t\rightarrow \infty$, $\nabla_{\bm{\beta}} (-L(\beta))+\nabla_{\bm{\beta}} \log(q(\beta))=0$, which implies $q(\beta)\propto e^{L(\beta)}$. In other words, $\lim_{t\rightarrow \infty}q(t, \beta)\propto e^{L(\beta)}$, i.e. $q(t, \beta)$ gradually converges to the Bayesian posterior $e^{L(\beta)}$.

% \subsection{Stochastic Approximation}
\subsection{Special Case: Robbins–Monro Algorithm}
Robbins–Monro algorithm is the first stochastic approximation algorithm to deal with the root finding problem which also applies to the stochastic optimization problem. Given the random output of $H(\bm{\theta}, \bm{\beta})$ with respect to $\bm{\beta}$, our goal is to find the equilibrium $\bm{\theta}^*$ such that
\begin{equation}
\begin{split}
\label{sa00}
h(\bm{\theta}^*)&=\E[H(\bm{\theta}^*, \bm{\beta})]=\int H(\bm{\theta}^*, \bm{\beta})f_{\bm{\theta}*}(d\bm{\beta})=0,
\end{split}
\end{equation}
% page 244 eq 1.10.3
where $\E_{\bm{\theta}^*}$ denotes the expectation with respect to the distribution of $\bm{\beta}$ given $\bm{\theta}^*$. To implement the Robbins–Monro Algorithm, we can generate iterates as follows\footnote{We change the notation a little bit, where $\bm{\beta}_k\in \mathbb{R}^d$ and $\bm{\theta}_k$ are the parameters at the $k$-th iteration.}:
\begin{itemize}
\item[(1)] Sample $\bm{\beta}_{k+1}$ from the invariant distribution $f_{\bm{\theta}_{k}}(\bm{\beta})$,

\item[(2)] Update $\bm{\theta}_{k+1}=\bm{\theta}_{k}+\omega_{k+1} H(\bm{\theta}_{k}, \bm{\beta}_{k+1}).$
\end{itemize}

Note that in this algorithm, $H(\bm{\theta}, \bm{\beta})$ is the unbiased estimator of $h(\bm{\theta})$, that is for $k\in \mathrm{N}^+$, we have
\begin{equation}
\label{sa11}
\E_{\bm{\theta}_{k}}[H(\bm{\theta}_{k}, \bm{\beta}_{k+1})-h(\bm{\theta}_{k})|\mathcal{F}_{k}]=0.
\end{equation}
If there exists an antiderivative $Q(\bm{\theta}, \bm{\beta})$ that satisfies $H(\bm{\theta}, \bm{\beta})= \nabla_{\bm{\theta}}Q(\bm{\theta}, \bm{\beta})$ and $E_{\bm{\theta}}[Q(\bm{\theta}, \bm{\beta})]$ is concave, it is equivalent to solving the stochastic optimization problem $\max_{\bm{\theta}\in\bm{\Theta}} E_{\bm{\theta}}[Q(\bm{\theta}, \bm{\beta})]$.

\subsection{General Stochastic Approximation}
The stochastic approximation algorithm is an iterative recursive algorithm consisting of two steps:
\begin{itemize}
\item[(1)] Sample $\bm{\beta}_{k+1}$ from the transition kernel  $\mathrm{\Pi}_{\bm{\theta_{k}}}(\bm{\beta}_{k}, \cdot)$, which admits $f_{\bm{\theta}_{k}}(\bm{\beta})$ as
the invariant distribution,

\item[(2)] Update $\bm{\theta}_{k+1}=\bm{\theta}_{k}+\omega_{k+1} H(\bm{\theta}_{k}, \bm{\beta}_{k+1})$.
\end{itemize}

The general stochastic approximation \citep{Albert90} differs from the Robbins-Monro algorithm in that sampling $x$ from a transition kernel  instead of a distribution introduces a Markov state-dependent noise $H(\theta_k, x_{k+1})-h(\theta_k)$. 

% If we can directly sample $\bm{\beta}_{k+1}$ from the $\mu_{\bm{\theta}}(\bm{\beta})$, then the algorithm is called the Robbins–Monro algorithm to deal with the root finding problem.

% Then we have $\E_{\bm{\theta}_{k}}[H(\bm{\theta}_{k}, \bm{\beta}_{k+1})|\mathcal{F}_{k}]=\int H(\bm{\theta}_{k}, \bm{\beta})\mathrm{\Pi}_{\bm{\theta}}(\bm{\beta}_{k}, d\bm{\beta})$. 

\section{Convergence Analysis}
\label{converge}
\subsection{Convergence of Hidden Variables}
The stochastic gradient Langevin Dynamics with a stochastic approximation adaptation (SGLD-SA) is a mixed half-optimization-half-sampling algorithm to handle complex Bayesian posterior with latent variables, e.g. the conjugate spike-slab hierarchical prior formulation. Each iteration of the algorithm consists of the following steps:
\begin{itemize}
\item[(1)] Sample $\bm{\beta}_{k+1}$ using SGLD based on $\bm{\theta}_{k}$, i.e.  
\begin{equation}
\begin{split}
\label{eqn-sample}
\bm{\beta}_{k+1}=\bm{\beta}_{k}+\epsilon\nabla_{\bm{\beta}}\tilde L(\bm{\beta}_{k}, \bm{\theta}_{k}) + \sqrt{{2\epsilon\tau^{-1}}}\bm{\eta}_{k},
\end{split}
\end{equation}
where $\bm{\eta}_{k}\sim\mathcal{N}(0, \bm{I})$;
\item[(2)] Optimize $\bm{\theta}_{k+1}$ from the following recursion
\end{itemize}
\begin{equation}
\begin{split}
\label{eqn-optimize}
\bm{\theta}_{k+1}&=\bm{\theta}_{k}+\omega_{k+1} \left(g_{\bm{\theta}_{k}}(\bm{\beta}_{k+1})-\bm{\theta}_{k}\right)\\
&=(1-\omega_{k+1})\bm{\theta}_{k}+\omega_{k+1}\bm{g}_{\bm{\theta}_{k}}(\bm{\beta}_{k+1}),\\
% &=\bm{\theta}_{k}+\omega_{k+1} H(\bm{\theta}_{k}, \bm{\beta}_{k+1})+\omega_{k+1}\bm{\bm{\xi}}_{k},\\
\end{split}
\end{equation}
where $\bm{g}_{\bm{\theta}_{k}}(\cdot)$ is some mapping to derive the optimal $\bm{\theta}$ based on the current $\bm{\beta}$. 

\textbf{Remark}: Define ${H(\bm{\theta}_{k}, \bm{\beta}_{k+1})=g_{\bm{\theta}_{k}}(\bm{\beta}_{k+1})-\bm{\theta}_{k}}$. In this formulation, our target is to find $\bm{\theta}$ that solves $h(\bm{\theta})=\E[H(\bm{\theta}, \bm{\beta})]=0$.
% , where $h(\bm{\theta}):=\int H(\bm{\theta}, \bm{\beta})f_{\bm{\theta}}(d\bm{\beta})$. $\E_{\bm{\theta}_{k}}[H(\bm{\theta}_{k}, \bm{\beta}_{k+1})|\mathcal{F}_{k}]$ converges to $h(\bm{\theta}_{k})$ as $k\rightarrow \infty$ , this algorithm falls to the category of the general stochastic approximation.

\subsubsection*{General Assumptions}
To provide the $L_2$ upper bound for SGLD-SA, we first lay out the following assumptions:
\begin{assumption}[Step size and Stability]
\label{ass1}
$\{\omega_{k}\}_{k\in \mathrm{N}}$ is a positive decreasing sequence of real numbers such that
\begin{equation}
\begin{split}
\label{a1}
&\omega_{k}\rightarrow 0, \ \ \sum_{k=1}^{\infty} \omega_{k}=+\infty.
\end{split}
\end{equation}
There exist $\delta>0$ and $\bm{\theta}^*$ such that for $\bm{\theta}\in \bm{\Theta}$: \footnote{$\|\cdot\|$ is short for $\|\cdot\|_2$}
\begin{equation}
\begin{split}
\label{a6}
\langle \bm{\theta}-\bm{\theta}^*, h(\bm{\theta})\rangle&\leq -\delta \|\bm{\theta}-\bm{\theta^*}\|^2,\\
\end{split}
\end{equation}
with additionally
\begin{equation}
\begin{split}
\label{a6-2}
\lim_{k\rightarrow \infty} \inf 2\delta  \dfrac{\omega_{k}}{\omega_{k+1}}+\dfrac{\omega_{k+1}-\omega_{k}}{{\omega_{k+1}^2}}>0.
\end{split}
\end{equation}
Then for any $\alpha \in (0, 1]$ and suitable $A$ and $B$, a practical $\omega_{k}$ can be set as
\begin{equation}
\begin{split}
\label{a6-3}
\omega_{k}=A(k+B)^{-\alpha}
\end{split}
\end{equation}
\end{assumption}

\begin{assumption}[Smoothness]
\label{ass_2_1}
$L(\bm{\beta}, \bm{\theta})$ is $M$-smooth with $M>0$, i.e. for any $\bm{\beta}, \bm{\iota}\in \bm{B}$, $\bm{\theta}, \bm{\upsilon}\in\bm{\Theta}$.
\begin{equation}
\begin{split}
\label{ass_2_1_eq}
&\|\nabla_{\bm{\beta}} L(\bm{\beta}, \bm{\theta})-\nabla_{\bm{\beta}} L(\bm{\iota}, \bm{\upsilon})\|\leq M\|\bm{\beta}-\bm{\iota}\|+M\|\bm{\theta}-\bm{\upsilon}\|. 
\end{split}
\end{equation}
\end{assumption}

\begin{assumption}[Dissipativity] There exist constants $m>0, b\geq 0$, s.t. for all $\bm{\beta} \in \bm{\beta}$ and $\bm{\theta} \in \bm{\Theta}$, we have 
\label{ass_dissipative}
\begin{equation}
\label{eq:01}
\langle \nabla_{\bm{\beta}} L(\bm{\beta}, \bm{\theta}), \bm{\beta}\rangle\leq b-m\|\bm{\beta}\|^2.
\end{equation}
\label{as:2_2}
\end{assumption}
\begin{assumption}[Gradient condition] The stochastic noise $\bm{\chi}_{k}\in \bm{B}$, which comes from $\nabla_{\bm{\beta}}\tilde L(\bm{\beta}_{k}, \bm{\theta}_{k})-\nabla_{\bm{\beta}} L(\bm{\beta}_{k}, \bm{\theta}_{k})$, is a white noise or Martingale difference noise and is independent with each other.
\begin{equation}
\begin{split}
\label{eq:19}
\E[\bm{\chi}_{k}|\mathcal{\bm{F}}_{k}]&=0.\\
\end{split}
\end{equation}
The scale of the noise is bounded by
\begin{equation}
\begin{split}
\label{eq:19_11}
\E\|\bm{\chi}\|^2 \leq M^2 \E\|\bm{\beta}\|^2+M^2 \E\|\bm{\theta}\|^2+B^2 .\\
\end{split}
\end{equation}
for constants $M, B>0$.
\label{as:1}
\end{assumption}
In addition to the assumptions, we also assume the existence of Markov transition kernel, the proof goes beyond the scope of our paper.
\begin{proposition}
There exist constants $M, B>0$ such that 
\begin{equation}
\begin{split}
\label{g_bound}
\|g_{\bm{\theta}}(\bm{\beta})\|^2\leq M^2\|\bm{\beta}\|^2+B^2
\end{split}
\end{equation}
\begin{proof}
As shown in Eq.(12), Eq.(13) and Eq.(15) in the main body, $\bm{\rho}, \delta$ and $\bm{\kappa}$ are clearly bounded. It is also easy to verify that $\sigma$ in Eq.(14) in the main body satisfies (\ref{g_bound}). For convenience, we choose the same $M$ and $B$ (large enough) as in (\ref{eq:19_11}).
\end{proof}
\end{proposition}
\begin{proposition}
\label{lemma:1_1}
For any $\bm{\beta}\in\bm{B}$, it holds that
\begin{equation}
\begin{split}
\label{linearbounded}
\|\nabla_{\bm{\beta}} L(\bm{\beta}, \bm{\theta})\|^2\leq 3M^2\|\bm{\beta}\|^2+3M^2\|\bm{\theta}\|^2+3B^2
\end{split}
\end{equation}
for constants $M$ and $B$.
\begin{proof}
Suppose there is a minimizer $(\bm{\theta}^*, \bm{\beta}^*)$ such that $\nabla_{\bm{\beta}} L(\bm{\beta}^*, \bm{\theta}^*)=0$ and $\bm{\theta}^*$ has reached the stationary point, following Assumption \ref{ass_dissipative} we have,
\begin{equation*}
\begin{split}
\langle \nabla_{\bm{\beta}} L(\bm{\beta}^*, \bm{\theta}^*), \bm{\beta}^*\rangle\leq b-m\|\bm{\beta}^*\|^2.\\
\end{split}
\end{equation*}
Therefore, $\|\bm{\beta}^*\|^2\leq \frac{b}{m}$. Since $\bm{\theta}^*$ is the stationary point, $\bm{\theta}^*=(1-
\omega)\bm{\theta}^*+\omega g_{\bm{\theta}^*}(\bm{\beta}^*)$. By (\ref{g_bound}), we have $\|g_{\bm{\theta}^*}(\bm{\beta}^*)\|^2\leq M^2\|\bm{\beta}^*\|^2+B^2$, which implies that $\|\bm{\theta}^*\|^2 = \|g_{\bm{\theta}^*}(\bm{\beta}^*)\|^2\leq  M^2\|\bm{\beta}^*\|^2+B^2\leq \frac{b}{m} M^2 + B^2$. By the smoothness assumption \ref{ass_2_1}, we have
\begin{equation*}
\begin{split}
&\|\nabla_{\bm{\beta}} L(\bm{\beta}, \bm{\theta})\|\\
\leq &\|\nabla_{\bm{\beta}} L(\bm{\beta}^*, \bm{\theta}^*)\|+M\|\bm{\beta}-\bm{\beta}^*\|+M\|\bm{\theta}-\bm{\theta}^*\|\\
\leq & 0+M(\|\bm{\beta}\|+\sqrt{\frac{b}{m}}+\|\bm{\theta}\|+\|\bm{\theta}^*\|)\\
\leq & M\|\bm{\theta}\|+M\|\bm{\beta}\|+ M(\sqrt{\frac{b}{m}}+\sqrt{\frac{b}{m}M^2+B^2})\\
\leq & M\|\bm{\theta}\|+M\|\bm{\beta}\|+\bar B,\\
\end{split}
\end{equation*}
where $\bar B =M(\sqrt{\frac{b}{m}M^2+B^2}+\sqrt{\frac{b}{m}})$.
Therefore, $$\|\nabla_{\bm{\beta}} L(\bm{\beta}, \bm{\theta})\|^2\leq 3M^2\|\bm{\beta}\|^2+3M^2\|\bm{\theta}\|^2+3\bar B^2.$$

For notation simplicity, we can choose the same $B$ (large enough) to bound (\ref{eq:19_11}), (\ref{g_bound}) and (\ref{linearbounded}).
\end{proof}
\end{proposition}
\begin{lemma}[Uniform $L_2$ bound]
\label{lemma:1}
For all $\ 0<\epsilon<\operatorname{Re}(\tfrac{m-\sqrt{m^2-4M^2(M^2+1)}}{4M^2(M^2+1)})$, there exist $G, \overline{G}>0$ such that
$\sup \E\|\bm{\beta}_{k}\|^2 \leq G$ and $\sup \E\|\bm{\theta}_{k}\|^2 \leq \overline{G}$, where $G=\|\bm{\beta}_0\|^2 + \tfrac{1}{m}(b+2\epsilon B^2(M^2+1)+\tau d)$ and $\overline{G}=M^2G+B^2$.
\end{lemma}
\begin{proof}
	From (\ref{eqn-sample}), we have
	\begin{equation}
	\begin{split}
	\label{eq:21}
	&\E\|\bm{\beta}_{k+1}\|^2\\
	=&\E\left\|\bm{\beta}_{k}+\epsilon \nabla_{\bm{\beta}} \tilde L(\bm{\beta}_{k}, \bm{\theta}_{k})\right\|^2+2\tau\epsilon \E\|\bm{\eta}_{k}\|^2 +\sqrt{8\epsilon\tau}\E\langle\bm{\beta}_{k}+\epsilon \nabla_{\bm{\beta}} \tilde L(\bm{\beta}_{k}, \bm{\theta}_{k}), \bm{\eta}_{k}\rangle\\
	=&\E\left\|\bm{\beta}_{k}+\epsilon \nabla_{\bm{\beta}} \tilde L(\bm{\beta}_{k}, \bm{\theta}_{k})\right\|^2 + 2\tau \epsilon d, \\
	\end{split}
	\end{equation}
	$\newline$
 Moreover, the first item in (\ref{eq:21}) can be expanded to 
	\begin{equation}
	\begin{split}
	\label{eq:22}
	&\ \ \  
  \E \left\|\bm{\beta}_{k}+\epsilon \nabla_{\bm{\beta}} \tilde L(\bm{\beta}_{k}, \bm{\theta}_{k}) \right\|^2 \\
	& =\E\left\|\bm{\beta}_{k}+\epsilon\nabla_{\bm{\beta}} L(\bm{\beta}_{k}, \bm{\theta}_{k})\right\|^2+\epsilon^2\E\left\|\bm{\chi}_{k}\right\|^2-2\epsilon \E \left[ 
   \E \left( \langle \bm{\beta}_{k}+\epsilon \nabla_{\bm{\beta}} L(\bm{\beta}_{k}, \bm{\theta}_{k}), \bm{\chi}_{k}\rangle |
 \mathcal{F}_{k} \right)  \right]  \\
	&=\E\left\|\bm{\beta}_{k}+\epsilon \nabla_{\bm{\beta}} L(\bm{\beta}_{k}, \bm{\theta}_{k})\right\|^2+\epsilon^2\E\left\|\bm{\chi}_{k}\right\|^2,
	\end{split}
	\end{equation}
	where (\ref{eq:19}) is used to cancel the inner product item. 

	Turning to the first item of (\ref{eq:22}), the dissipivatity condition (\ref{eq:01}) and the boundness of $\nabla_{\bm{\beta}} L(\bm{\beta}, \bm{\theta})$ (\ref{linearbounded})  give us:
	\begin{equation}
	\begin{split}
	\label{eq:27}
	&\ \ \ \ \ \E\left\|\bm{\beta}_{k}+\epsilon \nabla_{\bm{\beta}} L(\bm{\beta}_{k}, \bm{\theta}_{k})\right\|^2\\
	&=\E\|\bm{\beta}_{k}\|^2+2\epsilon\E\langle \bm{\beta}_{k},\nabla_{\bm{\beta}} L(\bm{\beta}_{k}, \bm{\theta}_{k})\rangle+\epsilon^2\E\left\|\nabla_{\bm{\beta}} L(\bm{\beta}_{k}, \bm{\theta}_{k})\right\|^2\\
	&\leq \E\|\bm{\beta}_{k}\|^2+2\epsilon(b-m \E\|\bm{\beta}_{k}\|^2)+\epsilon^2(3M^2\E\|\bm{\beta}_{k}\|^2+3M^2\E\|\bm{\theta}_{k}\|^2+3B^2)\\
	&=(1-2\epsilon m + 3\epsilon^2 M^2)\E\|\bm{\beta}_{k}\|^2+2\epsilon b +3\epsilon^2 B^2+3\epsilon^2M^2\E\|\bm{\theta}_{k}\|^2.
	\end{split}
	\end{equation}

	By (\ref{eq:19_11}), the second item of (\ref{eq:22}) is bounded by
	\begin{equation}
	\begin{split}
	\label{eq:28}
	\E\|\bm{\chi}_k\|^2\leq M^2\E\|\bm{\beta}_{k}\|^2+M^2\E\|\bm{\theta}_{k}\|^2+B^2. \\
	\end{split}
	\end{equation}

	Combining (\ref{eq:21}), (\ref{eq:22}), (\ref{eq:27}) and (\ref{eq:28}), we have
	\begin{equation}
	\begin{split}
	\label{eq:29}
	&\E\|\bm{\beta}_{k+1}\|^2
	\leq (1-2\epsilon m+4\epsilon^2 M^2)\E\|\bm{\beta}_{k}\|^2+2\epsilon b+4\epsilon^2 B^2+4\epsilon^2M^2\E\|\bm{\theta}_{k}\|^2+2\tau \epsilon d.\\
	\end{split}
	\end{equation}

Next we use proof by induction to show for $k=1,2,\ldots, \infty$, $\E\|\bm{\beta}_{k}\|^2\leq G$, where 
\begin{equation}  
\label{eq:49}
G =\E\|\bm{\beta}_0\|^2+\dfrac{b+2\epsilon B^2(M^2+1)+\tau d}{m-2\epsilon M^2(M^2+1)}.  
\end{equation} 

First of all, the case of $k=0, 1$ is trivial. Then if we assume for each $k \in 2, 3,\ldots, t$, $\E\|\bm{\beta}_{k}\|^2\leq G$, $\E\|g(\bm{\beta}_{k})\|^2 \leq M^2 G + B^2$, $\E\|\bm{\theta}_{k-1}\|^2 \leq M^2 G + B^2$. It follows that, 
\begin{equation*} 
\begin{split} 
&\ \ \ \ \ \E\|\bm{\theta}_{k}\|^2= \E\|(1-\omega_{k})\bm{\theta}_{k-1}+\omega_{k} g(\bm{\beta}_{k})\|^2\\
&\leq (1-\omega_{k})^2\E\|\bm{\theta}_{k-1}\|^2+{\omega_{k}}^2\E\|g(\bm{\beta}_{k})\|^2+2(1-\omega_{k})\omega_{k}\E\langle\bm{\theta}_{k-1},g(\bm{\beta}_{k})\rangle\\
&\leq (1-\omega_{k})^2\E\|\bm{\theta}_{k-1}\|^2+{\omega_{k}}^2\E\|g(\bm{\beta}_{k})\|^2  +2(1-\omega_{k})\omega_{k}\sqrt{\E\|\bm{\theta}_{k-1}\|^2 \E\|g(\bm{\beta}_{k})\|^2}\\
&\leq (1-\omega_{k})^2(M^2G+B^2)+{\omega_{k}}^2 (M^2G+B^2) +2(1-\omega_{k})\omega_{k} (M^2G+B^2) \\
&=M^2G+B^2,
\end{split}
\end{equation*} 
Next, we proceed to prove $\E\|\bm{\beta}_{t+1}\|^2\leq G$  and $\E\|\bm{\theta}_{t+1}\|^2\leq M^2G + B^2$. Following (\ref{eq:29}), we have 
\begin{equation}
\begin{split}
\label{eq:30}
&\ \ \ \ \ \E\|\bm{\beta}_{t+1}\|^2\\
&\leq (1-2\epsilon m+4\epsilon^2 M^2)\E\|\bm{\beta}_{k}\|^2+2\epsilon b+4\epsilon^2 B+4\epsilon^2M^2\E\|\bm{\theta}_{k}\|^2+2\tau \epsilon d \\
&\leq (1-2\epsilon m+4\epsilon^2 M^2)G+2\epsilon b+4\epsilon^2 B+4\epsilon^2M^2(M^2G+B^2)+2\tau \epsilon d \\
&\leq \left(1-2\epsilon m+4\epsilon^2 M^2 (M^2+1)\right)G + 2\epsilon b+4\epsilon^2 B^2(M^2+1)+2\tau\epsilon d\\
\end{split}
\end{equation}
	
	 Consider the quadratic equation $1-2mx+4M^2(M^2+1)x^2=0$. 
 If $m^2-4M^2(M^2+1)\geq0$, then the smaller root is $\tfrac{m-\sqrt{m^2-4M^2(M^2+1)}}{4M^2(M^2+1)}$ 
 which is positive; otherwise the quadratic equation has no real solutions and is always positive.
 Fix $\epsilon \in \left(0, \operatorname{Re}\left(\tfrac{m-\sqrt{m^2-4M^2(M^2+1)}}{4M^2(M^2+1)}\right)\right)$ so that 
\begin{equation}
\label{eq:40}
0<1-2\epsilon m+4\epsilon^2 M^2(M^2+1)<1.
\end{equation}

With (\ref{eq:49}), we can further bound (\ref{eq:30}) as follows:
\begin{equation}
\begin{split}
\label{eq: 300}
&\ \ \ \ \ \E\|\bm{\beta}_{t+1}\|^2\\
&\leq  \left(1-2\epsilon m+4\epsilon^2 M^2 (M^2+1)\right)\left(\E\|\bm{\beta}_0\|^2+\mathbb{I}\right) + 2\epsilon b+4\epsilon^2 B^2(M^2+1)+2d\tau\epsilon \\
&= \left(1-2\epsilon m+4\epsilon^2 M^2 (M^2+1)\right)\E\|\bm{\beta}_0\|^2 +\mathbb{I}- \left(2\epsilon b+4\epsilon^2 B^2(M^2+1)+2d\tau\epsilon\right)\\
&\ \ \ +\left(2\epsilon b+4\epsilon^2 B^2(M^2+1)+2\epsilon \tau d\right)\\
&\leq \E\|\bm{\beta}_0\|^2 +\mathbb{I}\equiv G,
\end{split}
\end{equation}
where $\mathbb{I}=\dfrac{b+2\epsilon B^2(M^2+1)+d\tau}{m-2\epsilon M^2(M^2+1)}$, the second to the last inequality comes from (\ref{eq:40}).

$\newline$
Moreover, from (\ref{g_bound}), we also have
\begin{equation*}
\begin{split}
\E\|g(\bm{\beta}_{t+1})\|^2 & \leq M^2\E\|\bm{\beta}_{t+1}\|^2+B^2\leq M^2G+B^2,\\
\end{split}
\end{equation*}
\begin{equation*} 
\begin{split} 
&\ \ \  \E\|\bm{\theta}_{t+1}\|^2= \E\|(1-\omega_{t+1})\bm{\theta}_t+\omega_{t+1} g(\bm{\beta}_{t+1})\|^2\\
&\leq (1-\omega_{t+1})^2\E\|\bm{\theta}_t\|^2+\omega_{t+1}^2\E\|g(\bm{\beta}_{t+1})\|^2
+2(1-\omega_{t+1})\omega_{t+1}\E\langle\bm{\theta}_t,g(\bm{\beta}_{t+1})\rangle\\
&\leq (1-\omega_{t+1})^2\E\|\bm{\theta}_t\|^2+\omega_{t+1}^2\E\|g(\bm{\beta}_{t+1})\|^2+2(1-\omega_{t+1})\omega_{t+1}\sqrt{\E\|\bm{\theta}_t\|^2 \E\|g(\bm{\beta}_{t+1})\|^2}\\
&\leq (1-\omega_{t+1})^2(M^2G+B^2)+\omega_{t+1}^2 (M^2G+B^2) +2(1-\omega_{t+1})\omega_{t+1} (M^2G+B^2) \\
&=M^2G+B^2,
\end{split}
\end{equation*} 
Therefore, we have proved that for any $k \in 1,2,\ldots, \infty$, $\E\|\bm{\beta}_{k}\|^2$, $\E\|g(\bm{\beta}_{k})\|^2$ 
and $\E\|\bm{\theta}_{k}\|^2$ are bounded. Furthermore, we notice that $G$ can be unified to a constant $G=\E\|\bm{\beta}_0\|^2+\tfrac{1}{m}\left(b+2\epsilon B^2(M^2+1)+\tau d\right)$.
\end{proof}

\begin{assumption}[Solution of Poisson equation]
\label{ass_5}
For all $\bm{\theta}\in \bm{\Theta}$, $\bm{\beta}\in \bm{B}$ and a function $V(\bm{\beta})=1+\|\bm{\beta}\|^2$, there exists a function $\mu_{\bm{\theta}}(\cdot)$ on $\bm{B}$ that solves the Poisson equation $\mu_{\bm{\theta}}(\bm{\beta})-\mathrm{\Pi}_{\bm{\theta}}\mu_{\bm{\theta}}(\bm{\beta})=H(\bm{\theta}, \bm{\beta})-h(\bm{\theta})$, which follows that 
\begin{equation}
\begin{split}
\label{a4.ii}
&H(\bm{\theta}_{k}, \bm{\beta}_{k+1})=h(\bm{\theta}_{k})+\mu_{\bm{\theta}_{k}}(\bm{\beta}_{k+1})-\mathrm{\Pi}_{\bm{\theta}_{k}}\mu_{\bm{\theta}_{k}}(\bm{\beta}_{k+1}).
\end{split}
\end{equation}
Moreover, for all $\bm{\theta}, \bm{\theta}'\in \bm{\Theta}$ and $\bm{\beta}\in \bm{B}$, we have
% There exists constant $C_1$ and $C_2$ such that for all $\bm{\theta}, \bm{\theta}'\in \bm{\Theta}$ and we leave the relaxation of this assumption for future work.
\begin{equation*}
\begin{split}
\|\mu_{\bm{\theta}}(\bm{\beta})-\mu_{\bm{\theta'}}(\bm{\beta})\|&\lesssim \|\bm{\theta}-\bm{\theta}'\|V(\bm{\beta}) ,\\
\|\mu_{\bm{\theta}}(\bm{\beta})\|&\lesssim V(\bm{\beta})
\end{split}
\end{equation*}
\end{assumption}
Poisson equation has been widely in controlling the perturbations of adaptive algorithms \citep{Albert90, andrieu05, Liang10}. Following \citet{Albert90}, we only impose the 1st-order smoothness on the solution of Poisson equation, which is much weaker than the 4th-order smoothness used in proving the ergodicity theory \citep{mattingly10, Teh16}. The verification of the above assumption is left as future works.
\begin{proposition} The uniform $L_2$ bound in lemma \ref{lemma:1} further implies that $\E[V(\bm{\beta})]<\infty$. In what follows, there exists a constant $C$ such that 
\begin{equation}
\begin{split}
\label{a5}
\E[\|\mathrm{\Pi}_{\bm{\theta}}\mu_{\bm{\theta}}(\bm{\beta})-\mathrm{\Pi}_{\bm{\theta}'}\mu_{\bm{\theta'}}(\bm{\beta})\|]&\leq C\|\bm{\theta}-\bm{\theta}'\| ,\\
\E[\|\mathrm{\Pi}_{\bm{\theta}}\mu_{\bm{\theta}}(\bm{\beta})\|]&\leq C.
\end{split}
\end{equation}
\end{proposition}

\begin{proposition}
There exists a constant $C_1$ so that
\begin{equation}
\begin{split}
\label{bound2}
&\E[\|H(\bm{\theta}, \bm{\beta})\|^2] \leq C_1 (1+\|\bm{\theta}-\bm{\theta}^*\|^2)\\
\end{split}
\end{equation}
\begin{proof}
By (\ref{g_bound}), we have
\begin{equation*}
\begin{split}
\E\|g_{\bm{\theta}}(\bm{\beta})-\bm{\theta}\|^2&\leq 2\E\|g_{\bm{\theta}}(\bm{\beta})\|^2+2\|\bm{\theta}\|^2\leq 2(M^2\E\|\bm{\beta}\|^2+B^2)+2\|\bm{\theta}\|^2\\
\end{split}
\end{equation*}
Given the boundness of $\E\|\bm{\beta}\|^2$ and a large enough $C'=\max(2, 2(M^2\E\|\bm{\beta}\|^2+B^2))$, we have
\begin{equation*}
\begin{split}
&\E[\|H(\bm{\theta}, \bm{\beta})\|^2] \leq C' (1+\|\bm{\theta}\|^2) = C' (1+\|\bm{\theta}-\bm{\theta}^*+\bm{\theta}^*\|^2)\leq C_1 (1+\|\bm{\theta}-\bm{\theta}^*\|^2)\\
\end{split}
\end{equation*}
\end{proof}
\end{proposition}

Lemma 2 is a restatement of Lemma 25 (page 247) from \citet{Albert90}.
\begin{lemma}
\label{lemma:2}
Suppose $k_0$ is an integer which satisfies
% \begin{equation}
% \begin{split}
% % \label{lemma1:condition}
% 1-2\omega_{k+1}\delta-\omega_{k+1}^2C_1\geq 0\ \for\ \all\ k\geq k_0,
% \end{split}
% \end{equation}
with 
\begin{equation*}
\begin{split}
\inf_{k\geq k_0} \dfrac{\omega_{k+1}-\omega_{k}}{\omega_{k}\omega_{k+1}}+2\delta-\omega_{k+1}C_1>0.
\end{split}
\end{equation*}

Then for any $k>k_0$, the sequence $\{\Lambda_k^K\}_{k=k_0, \ldots, K}$ defined below is increasing %and bounded by % $2\omega_{k}$
\begin{equation}  
\left\{  
             \begin{array}{lr}  
             2\omega_{k}\prod_{j=k}^{K-1}(1-2\omega_{j+1}\delta+\omega_{j+1}^2C_1) & \text{if $k<K$},   \\  
              & \\
             2\omega_{k} &  \text{if $k=K$}.
             \end{array}  
\right.  
\end{equation} 
\end{lemma}

This following lemma is proved by \citet{Albert90} and we present it here for reader's convenience.

\begin{lemma}
\label{lemma:4}
There exist $\lambda_0$ and $k_0$ such that for all $\lambda\geq\lambda_0$ and $k\geq k_0$, the sequence $u_{k}=\lambda \omega_{k}$ satisfies
\begin{equation}
\begin{split}
\label{key_ieq}
u_{k+1}\geq& (1-2\omega_{k+1}\delta+{\omega_{k+1}^2} C_1)u_{k}+{\omega_{k+1}^2} \overline{C}_1.
\end{split}
\end{equation}

\begin{proof}
By replacing $u_{k}=\lambda \omega_{k}$ in ($\ref{key_ieq}$), we need to prove the following
\begin{equation}
\begin{split}
\label{lemma:loss_control}
\lambda \omega_{k+1}\geq& (1-2\omega_{k+1}\delta+{\omega_{k+1}^2} C_1)\lambda \omega_{k}+{\omega_{k+1}^2} \overline{C}_1.
\end{split}
\end{equation}

According to ($\ref{a6-2}$) in assumption \ref{ass1}, there exists a positive constant $C_+$ such that
\begin{equation}
\label{lr_condition}
 \omega_{k+1}-\omega_{k}+2 \omega_{k}\omega_{k+1} \geq C_+ \omega_{k+1}^2,
\end{equation}

Combining (\ref{lemma:loss_control}) and (\ref{lr_condition}), it suffices to show
\begin{equation}
\begin{split}
\label{loss_control-3}
\lambda (C_+-\omega_{k} C_1)\geq \overline{C}_1.
\end{split}
\end{equation}
Apparently, there exist $\lambda_0$ and $k_0$ such that for all $\lambda>\lambda_0$ and $k>k_0$, ($\ref{loss_control-3}$) holds.
\end{proof}
\end{lemma}

% \textbf{Theorem 1} ($L_2$ convergence rate).
% \begin{theorem}[$L_2$ convergence rate]
% Suppose that Assumptions $\ref{ass1}$-$\ref{ass_5}$ hold, there exists a constant $\lambda$ such that
% \begin{equation*}
% \begin{split}
% \E\left[\|\bm{\theta}_{k}-\bm{\theta}^*\|^2\right]\leq \lambda\omega_{k},
% \end{split}
% \end{equation*}
% \end{theorem}

\textbf{Theorem 1} ($L_2$ convergence rate). \textit{Suppose that Assumptions $\ref{ass1}$-$\ref{ass_5}$ hold, there exists a constant $\lambda$ such that}
\begin{equation*}
\begin{split}
\E\left[\|\bm{\theta}_{k}-\bm{\theta}^*\|^2\right]\leq \lambda\omega_{k},
\end{split}
\end{equation*}

\begin{proof}
Denote $\bm{T}_{k}=\bm{\theta}_{k}-\bm{\theta}^*$,   
 with the help of (\ref{eqn-optimize}) and Poisson equation (\ref{a4.ii}), we deduce that 
\begin{equation*}
\begin{split}
&\|\bm{T}_{k+1}\|^2\\
=&\|\bm{T}_{k}\|^2+{\omega_{k+1}^2} \| H(\bm{\theta}_{k}, \bm{\beta}_{k+1}) \|^2+2\omega_{k+1} \langle \bm{T}_{k}, H(\bm{\theta}_{k}, \bm{\beta}_{k+1})\rangle\\
=&\|\bm{T}_{k}\|^2+{\omega_{k+1}^2} \| H(\bm{\theta}_{k}, \bm{\beta}_{k+1}) \|^2+2\omega_{k+1} \langle \bm{T}_{k}, h(\bm{\theta}_{k})\rangle+2\omega_{k+1} \langle \bm{T}_{k}, \mu_{\bm{\theta}_{k}}(\bm{\beta}_{k+1})-\mathrm{\Pi}_{\bm{\theta}_{k}}\mu_{\bm{\theta}_{k}}(\bm{\beta}_{k+1})\rangle\\
=&\|\bm{T}_{k}\|^2+\text{D1}+\text{D2}+\text{D3}.\\
\end{split}
\end{equation*}

First of all, according to (\ref{bound2}) and (\ref{a6}), we have
\begin{align*}
{\omega_{k+1}^2} \E[\| H(\bm{\theta}_{k}, \bm{\beta}_{k+1}) \|^2]&\leq {\omega_{k+1}^2} C_1(1+\|\bm{T}_{k}\|^2), \tag{D1}\\
2\omega_{k+1}\E[\langle \bm{T}_{k}, h(\bm{\theta}_{k})\rangle] &\leq - 2\omega_{k+1}\delta\|\bm{T}_{k}\|^2,\tag{D2}\\
\end{align*}
Conduct the decomposition of \text{D3} similar to Theorem 24 (p.g. 246) from \citet{Albert90} and Lemma A.5 \citep{Liang10}.
\begin{equation*}
\begin{split}
&\mu_{\bm{\theta}_{k}}(\bm{\beta}_{k+1})-\mathrm{\Pi}_{\bm{\theta}_{k}}\mu_{\bm{\theta}_{k}}(\bm{\beta}_{k+1}) \\
=&\underbrace{\mu_{\bm{\theta}_{k}}(\bm{\beta}_{k+1})-\mathrm{\Pi}_{\bm{\theta}_{k}}\mu_{\bm{\theta}_{k}}(\bm{\beta}_{k})}_{\text{D3-1}}+ \underbrace{\mathrm{\Pi}_{\bm{\theta}_{k}}\mu_{\bm{\theta}_{k}}(\bm{\beta}_{k})- \mathrm{\Pi}_{\bm{\theta}_{k-1}}\mu_{\bm{\theta}_{k-1}}(\bm{\beta}_{k})}_{\text{D3-2}}+ \underbrace{\mathrm{\Pi}_{\bm{\theta}_{k-1}}\mu_{\bm{\theta}_{k-1}}(\bm{\beta}_{k})- \mathrm{\Pi}_{\bm{\theta}_{k}}\mu_{\bm{\theta}_{k}}(\bm{\beta}_{k+1})}_{\text{D3-3}}.\\
\end{split}
\end{equation*}

(i) $\mu_{\bm{\theta}_{k}}(\bm{\beta}_{k+1})-\mathrm{\Pi}_{\bm{\theta}_{k}}\mu_{\bm{\theta}_{k}}(\bm{\beta}_{k})$ forms a martingale difference sequence such that 
$$\E\left[\mu_{\bm{\theta}_{k}}(\bm{\beta}_{k+1})-\mathrm{\Pi}_{\bm{\theta}_{k}}\mu_{\bm{\theta}_{k}}(\bm{\beta}_{k})|\mathcal{F}_{k}\right]=0. \eqno{(\text{D3-1})}$$

(ii) From Lemma \ref{lemma:1}, we have that $\E[\|\bm{T_k}\|]$ is bounded. In addition, following Cauchy–Schwarz inequality and the regularity conditions in (\ref{a5}), there exists a positive constant $C_2$ such that
\begin{align*}
&\E\left[2\omega_{k+1}\langle\bm{T}_{k},\mathrm{\Pi}_{\bm{\theta}_{k}}\mu_{\bm{\theta}_{k}}(\bm{\beta}_{k})- \mathrm{\Pi}_{\bm{\theta}_{k-1}}\mu_{\bm{\theta}_{k-1}}(\bm{\beta}_{k})\rangle\right]\leq \omega_{k+1}^2 C_2\tag{D3-2}.
\end{align*}

(iii) $\text{D3-3}$ can be further decomposed to $\text{D3-3a}$ and $\text{D3-3b}$
\begin{equation*}
\small
\begin{split}
&\ \ \ \langle \bm{T}_{k},\mathrm{\Pi}_{\bm{\theta}_{k-1}}\mu_{\bm{\theta}_{k-1}}(\bm{\beta}_{k})- \mathrm{\Pi}_{\bm{\theta}_{k}}\mu_{\bm{\theta}_{k}}(\bm{\beta}_{k+1})\rangle\\
&=\left(\langle \bm{T}_{k}, \mathrm{\Pi}_{\bm{\theta}_{k-1}}\mu_{\bm{\theta}_{k-1}}(\bm{\beta}_{k}) \rangle- \langle \bm{T}_{k+1}, \mathrm{\Pi}_{\bm{\theta}_{k}}\mu_{\bm{\theta}_{k}}(\bm{\beta}_{k+1})\rangle\right) +\left(\langle \bm{T}_{k+1}, \mathrm{\Pi}_{\bm{\theta}_{k}}\mu_{\bm{\theta}_{k}}(\bm{\beta}_{k+1})\rangle-\langle \bm{T}_{k}, \mathrm{\Pi}_{\bm{\theta}_{k}}\mu_{\bm{\theta}_{k}}(\bm{\beta}_{k+1})\rangle\right)\\
&=\underbrace{(z_{k}-z_{k+1})}_{\text{D3-3a}}+\underbrace{\langle \bm{T}_{k+1}-\bm{T}_{k}, \mathrm{\Pi}_{\bm{\theta}_{k}}\mu_{\bm{\theta}_{k}}(\bm{\beta}_{k+1})\rangle}_{\text{D3-3b}}.\\
\end{split}
\end{equation*}

where $z_{k}=\langle \bm{T}_{k}, \mathrm{\Pi}_{\bm{\theta}_{k-1}}\mu_{\bm{\theta}_{k-1}}(\bm{\beta}_{k})\rangle$. Combining (\ref{eqn-optimize}), (\ref{g_bound}), Lemma \ref{lemma:1} and Cauchy–Schwarz inequality, there exists a constant $C_3$ such that 
\begin{align*}
&\E\left[2\omega_{k+1}\langle \bm{T}_{k+1}-\bm{T}_{k}, \mathrm{\Pi}_{\bm{\theta}_{k}}\mu_{\bm{\theta}_{k}}(\bm{\beta}_{k+1})\rangle\right] \leq \omega_{k+1}^2 C_3
\end{align*}
Finally, add all the items \text{D1}, \text{D2} and \text{D3} together, for some $\overline{C}_1 = C_1+C_2 + C_3$, we have 
\begin{equation*}
\begin{split}
&\E\left[\|\bm{T}_{k+1}\|^2\right]
\leq (1-2\omega_{k+1}\delta+{\omega_{k+1}^2} C_1)\E\left[\|\bm{T}_{k}\|^2\right]+{\omega_{k+1}^2} \overline{C}_1+2\omega_{k+1}\E[z_{k}-z_{k+1}].
\end{split}
\end{equation*}
Moreover, from (\ref{a5}), there exists a constant $C_4$ such that
\begin{equation}
\begin{split}
\label{condition:z}
\E[|z_{k}|]\leq C_4.
\end{split}
\end{equation}
Lemma 4 is a restatement of Lemma 26 (page 248) from \citet{Albert90}.
\begin{lemma}
\label{lemma:3-all}
Let $\{u_{k}\}_{k\geq k_0}$ as a sequence of real numbers such that for all $k\geq k_0$, some suitable constants  $\overline{C}_1$ and $C_1$
\begin{equation}
\begin{split}
\label{lemma:3-a}
u_{k+1}\geq &u_{k}\left(1-2\omega_{k+1}\delta+{\omega_{k+1}^2} C_1\right)+{\omega_{k+1}^2} \overline{C}_1,
\end{split}
\end{equation}
$\newline$
and assume there exists such $k_0$ that 
\begin{equation}
\begin{split}
\label{lemma:3-b}
\E\left[\|\bm{T}_{k_0}\|^2\right]\leq u_{k_0}.
\end{split}
\end{equation}
Then for all $k> k_0$, we have
\begin{equation*}
\begin{split}
\E\left[\|\bm{T}_{k}\|^2\right]\leq u_{k}+\sum_{j=k_0+1}^{k}\Lambda_j^k (z_{j-1}-z_{j}).
\end{split}
\end{equation*}
\end{lemma}

\textbf{Proof of Theorem 1 (Continued)}. From Lemma $\ref{lemma:4}$, we can choose $\lambda_0$ and $k_0$ which satisfy the conditions ($\ref{lemma:3-a}$) and ($\ref{lemma:3-b}$)
\begin{align*}
\E[\|\bm{T}_{k_0}\|^2]\leq u_{k_0}=\lambda_0 \omega_{k_0}.
\end{align*}

From Lemma $\ref{lemma:3-all}$, it follows that for all $k>k_0$
\begin{equation}
\begin{split}
\label{eqn:9}
\E\left[\|\bm{T}_{k}\|^2\right]\leq u_{k}+\E\left[\sum_{j=k_0+1}^{k}\Lambda_j^k \left(z_{j-1}-z_{j}\right)\right].
\end{split}
\end{equation}
From ($\ref{condition:z}$) and the increasing property of $\Lambda_{j}^k$ in Lemma $\ref{lemma:2}$, we have
\begin{equation}
\begin{split}
\label{eqn:10}
&\E\left[\left|\sum_{j=k_0+1}^{k} \Lambda_j^k\left(z_{j-1}-z_{j}\right)\right|\right]\\
=&\E\left[\left|\sum_{j=k_0+1}^{k-1}(\Lambda_{j+1}^k-\Lambda_j^k)z_{j}-2\omega_{k}z_{k}+\Lambda_{k_0+1}^k z_{k_0}\right|\right]\\
\leq& \sum_{j=k_0+1}^{k-1}(\Lambda_{j+1}^k-\Lambda_j^k)C_4+\E[|2\omega_{k} z_{k}|]+\Lambda_k^k C_4\\
\leq& (\Lambda_k^k-\Lambda_{k_0}^k)C_4+\Lambda_k^k C_4+\Lambda_k^k C_4\\
\leq& 3\Lambda_k^k C_4=6C_4\omega_{k}.
\end{split}
\end{equation}

Therefore, given the sequence $u_{k}=\lambda_0 \omega_{k}$ that satisfies conditions ($\ref{lemma:3-a}$), ($\ref{lemma:3-b}$) and Lemma $\ref{lemma:3-all}$, for any $k>k_0$, from ($\ref{eqn:9}$) and ($\ref{eqn:10}$), we have
\begin{equation*}
\E[\|\bm{T}_{k}\|^2]\leq u_{k}+3C_4\Lambda_k^k=\left(\lambda_0+6C_4\right)\omega_{k}=\lambda \omega_{k},
\end{equation*}
where $\lambda=\lambda_0+6C_4$.
\end{proof}

\subsection{Weak Convergence of Samples}

In statistical models with latent variables, the gradient is often biased due to the use of stochastic approximation. Langevin Monte Carlo with inaccurate gradients has been studied by \citet{Chen15, Dalalyan18}, which are helpful to prove the weak convergence of samples. Following theorem 2 in \citet{Chen15}, we have
% For instance, \cite{Dalalyan18} studied the effect of bounded bias to Langevin Monte Carlo and 
% \begin{corollary}

\textbf{Corollary 1.} \textit{Under assumptions in Appendix B.1 and the assumption 1 (smoothness and boundness on the solution functional) in \citet{Chen15}, the distribution of $\bm{\beta}_{k}$ converges weakly to the target posterior as $\epsilon\rightarrow 0$ and $k\rightarrow \infty$.}
% \end{corollary}
\begin{proof}
% Given constant stepsize $\epsilon$, we ignore the subscript of $\epsilon$ and consider the case of $\tau=1$. 

% The result for $\tau \neq 1$ is equivalent to rescaling the noise $\eta$ to $\eta/\tau$ and $L(\cdot)$ to $\tau L(\cdot)$ as proved in \citet{Maxim17, Yuchen17, Xu18}.

Since $\bm{\theta}_{k}$ converges to $\bm{\theta}^*$ in SGLD-SA under assumptions in Appendix B.1 and the gradient is M-smooth (\ref{ass_2_1_eq}), we decompose the stochastic gradient $\nabla_{\bm{\beta}}\tilde L(\bm{\beta}_{k}, \bm{\theta}_{k})$ as $\nabla_{\bm{\beta}} L(\bm{\beta}_{k}, \bm{\theta}^*)+\bm{\xi}_{k}+\mathcal{O}(k^{-\alpha})$, where $\nabla_{\bm{\beta}} L(\bm{\beta}_{k}, \bm{\theta}^*)$ is the exact gradient, $\bm{\xi}_k$ is a zero-mean random vector, $\mathcal{O}(k^{-\alpha})$ is the bias term coming from the stochastic approximation and $\alpha\in (0, 1]$ is used to guarantee the consistency in theorem 1. Therefore, Eq.(\ref{eqn-sample}) can be written as
\begin{equation}
\begin{split}
\bm{\beta}_{k+1}&=\bm{\beta}_{k} + \epsilon_k \left(\nabla_{\bm{\beta}} L(\bm{\beta}_{k}, \bm{\theta}^*)+\bm{\xi}_{k}+\mathcal{O}(k^{-\alpha})\right)+\sqrt{2\epsilon_k}\bm{\eta}_{k},\ \where\ \bm{\eta}_{k}\sim \mathcal{N}(0, \bm{I}). \\
\end{split}
\end{equation}
Following a similar proof in \citet{Chen15}, it suffices to show that $\sum_{k=1}^{K} k^{-\alpha}/K\rightarrow 0$ as $K\rightarrow\infty$, which is obvious. Therefore, the distribution of $\bm{\beta}_{k}$ converges weakly to the target distribution as $\epsilon\rightarrow 0$ and $k\rightarrow \infty$.
\end{proof}

\section{Simulation of Large-p-Small-n Logistic Regression}
\label{logistic_reg}
Now we conduct the experiments on binary logistic regression. The setup is similar as before, except $n$ is set to 500, $\Sigma_{i,j}=0.3^{|i-j|}$ and $\bm{\eta}\sim \mathcal{N}(0,\bm{I}/2)$. We set the learning rate for all the three algorithms to $0.001\times k^{-\frac{1}{3}}$ and step size $\omega_{k}$ to $10\times (k+1000)^{-0.7}$. The binary response values are simulated from $\textbf{Bernoulli}(p)$ where $p=1/(1+e^{-X\bm{\beta}-\bm{\eta}})$. As shown in Fig.\ref{fig:Logistic regression}: SGLD fails in selecting the right variables and overfits the data; both SGLD-EM and SGLD-SA choose the right variables. However, SGLD-EM converges to a poor local optimum by mistakenly using $L_1$ norm to regularize all the variables, leading to a large shrinkage effect on $\bm{\beta}_{1:3}$. By contrast, SGLD-SA successfully updates the latent variables and regularize $\beta_{1:3}$ with $L_2$ norm, yielding a better parameter estimation for $\beta_{1:3}$ and a stronger regularization for $\beta_{4-1000}$. Table.\ref{Logistic_regression_UQ_test} illustrates that SGLD-SA consistently outperforms the other methods and is robust to different initializations. We observe that SGLD-EM sometimes performs as worse as SGLD when $v_0=0.001$, which indicates that the EM-based variable selection is not robust in the stochastic optimization of the latent variables.

\begin{table*}[t]  % Linear regression with UQ pars
\caption{Predictive errors in logistic regression based on a test set considering different $v_0$ and $\sigma$}
\label{Logistic_regression_UQ_test}
\begin{center}
\begin{small}
\begin{sc}
\begin{tabular}{lccccr}
\toprule
MAE / MSE & $v_0$=$0.01,\sigma$=1 & $v_0$=$0.001,\sigma$=1 & $v_0$=$0.01, \sigma$=2  &  $v_0$=$0.001, \sigma$=2 &\\
\midrule %% testing
SGLD-SA     & \textbf{0.177} / \textbf{0.108} & \textbf{0.188} / \textbf{0.114} &  \textbf{0.182} / \textbf{0.116} & \textbf{0.187} / \textbf{0.113} &\\
SGLD-EM     & 0.207 / 0.131 & 0.361 / 0.346 & 0.204 / 0.132  & 0.376 / 0.360 \\
SGLD        & 0.295 / 0.272  & 0.335 / 0.301 & 0.350 / 0.338  & 0.337 / 0.319  &\\
\bottomrule
\end{tabular}
\end{sc}
\end{small}
\end{center}
\vskip -0.1in
\end{table*}

\begin{figure}[!ht]
\centering
\includegraphics[scale=0.25]{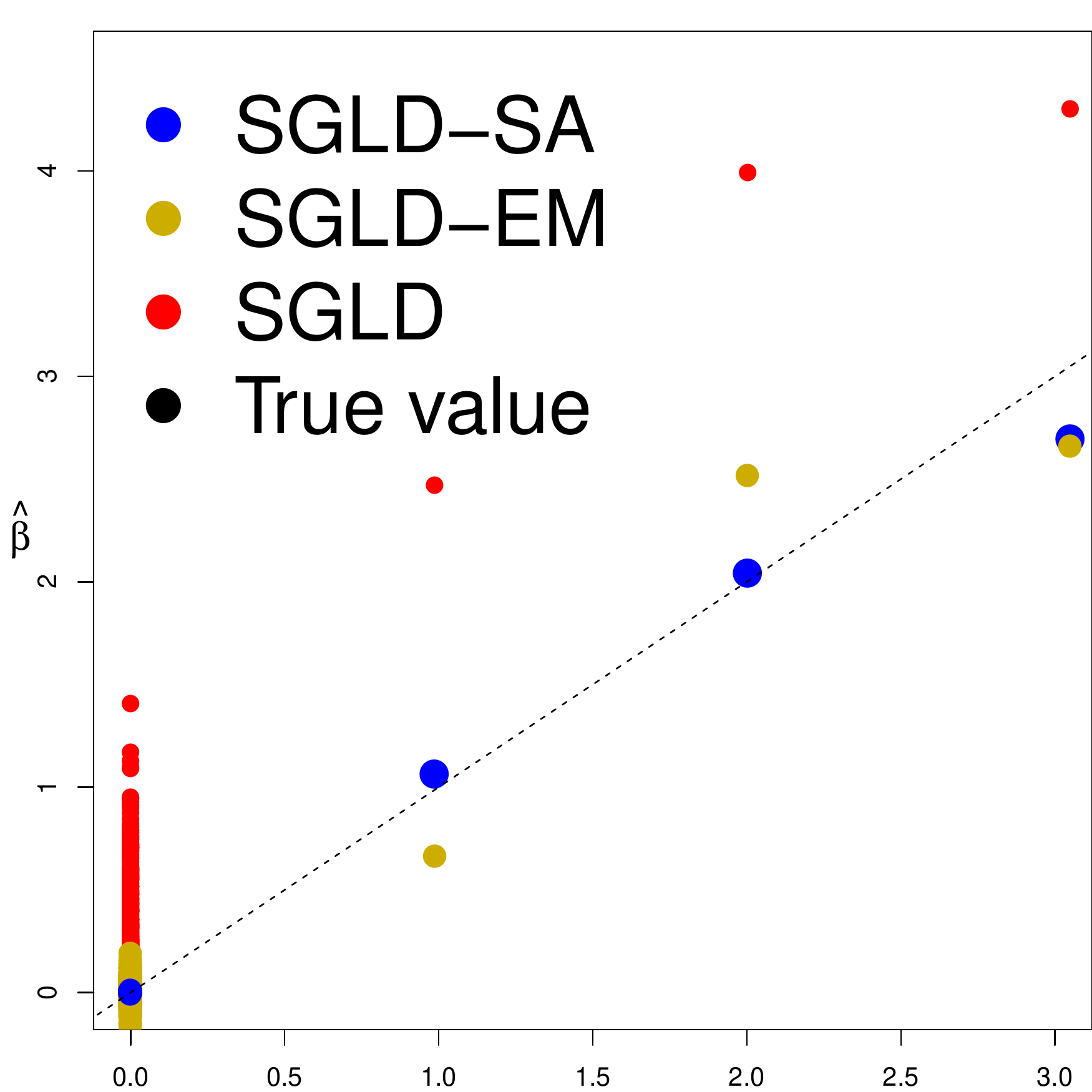}
\caption{Logistic regression simulation when $v_0=0.1$ and $\sigma=1$}
\label{fig:Logistic regression}
\end{figure}

\section{Regression on UCI datasets}
\label{reg_UCI}
We further evaluate our model on five \textbf{UCI} regression datasets and show the results in Table~\ref{UCI}. Following \cite{Jose_adam_15}, we randomly sample 90\% of each dataset for training and leave the rest for testing. We run 20 experiments for each setup with fixed random seeds and report the averaged error rate. Feature normalization is applied in the experiments.
The model is a simple MLP with one hidden layer of 50 units. We set the batch size to 50, the training epoch to 200, the learning rate to 1e-5 and the default $L_2$ to 1e-4. For SGHMC-EM and SGHMC-SA, we apply the SSGL prior on the BNN weights (excluding biases) and fix $a,\nu,\lambda=1$, $b, v_1,\sigma=10$ and $\delta=0.5$. We fine-tune the initial temperature $\tau$ and $v_0$. As shown in Table~\ref{UCI}, SGHMC-SA outperforms all the baselines. Nevertheless, without smooth adaptive update, SGHMC-EM often performs worse than SGHMC. While with simulated annealing where $\tau^{(k)}=\tau\times 1.003^k$, we observe further improved performance in most of the cases.

\begin{table}[!htb]
  \centering
  \small
%   \begin{tabular}{0.8\textwidth}{c|ccccccc}
  \begin{tabular}{c|ccccccc}
    \toprule
    % \multicolumn{4}{c}{Criteo Dataset}                   \\
    % \midrule
    Dataset &   Boston & Yacht &  Energy & Wine & Concrete \\
    Hyperparameters &  1/0.1 & 1/0.1 & 0.1/0.1 & 0.5/0.01 & 0.5/0.07\\
     \midrule
    SGHMC  & 2.783$\pm$0.109 & 0.886$\pm$0.046 & 1.983$\pm$0.092 & 0.731$\pm$0.015 & 6.319$\pm$0.179 \\
    A-SGHMC  &  2.848$\pm$0.126 & 0.808$\pm$0.048 & 1.419$\pm$0.067 & 0.671$\pm$0.019 & 5.978$\pm$0.166 \\
    \midrule
     SGHMC-EM  & 2.813$\pm$0.140 & 0.823$\pm$0.053 & $2.077\pm$0.108 & 0.729$\pm$0.018 & 6.275$\pm$0.169 \\
     A-SGHMC-EM  & 2.767$\pm$0.154 & 0.815$\pm$0.052 & 1.435$\pm$0.069 & 0.627$\pm$0.008 &  5.762$\pm$0.156 \\
     \midrule
     SGHMC-SA  & \textbf{2.779$\pm$0.133} & \textbf{0.789$\pm$0.050} & \textbf{1.948$\pm$0.081} & \textbf{0.654$\pm$0.010} & \textbf{6.029$\pm$0.131}\\
     A-SGHMC-SA  & \textbf{2.692$\pm$0.120} & \textbf{0.782$\pm$0.052} & \textbf{1.388$\pm$0.052} & \textbf{0.620$\pm$0.008} & \textbf{5.687$\pm$0.142} \\
    \bottomrule
  \end{tabular}
  \caption{Average performance and standard deviation of Root Mean Square Error, where $\tau$ denotes the initial inverse temperature and $v_0$ is a hyperparameter in the SSGL prior (Hyperparameters $\tau/v_0$). 
  }
  \label{UCI}
\end{table}

\section{Experimental Setup}
\label{exp_setup}
\subsection{Network Architecture}

The first DNN we use is a standard 2-Conv-2-FC CNN: it has two convolutional layers with a 2 $\times$ 2 max pooling after each layer and two fully-connected layers. The filter size in the convolutional layers is 5 $\times$ 5 and the feature maps are set to be 32 and 64, respectively  \citep{Jarrett09}. The fully-connected layers (FC) have 200 hidden nodes and 10 outputs. We use the rectified linear unit (ReLU) as activation function between layers and employ a cross-entropy loss. 

The second DNN is a 2-Conv-BN-3-FC CNN: it has two convolutional layers with a $2\times2$ max pooling after each layer and three fully-connected layers with batch normalization applied to the first FC layer. The filter size in the convolutional layers is $4\times4$ and the feature maps are both set to 64. We use $256\times64\times10$ fully-connected layers. 

\subsection{Data Augmentation}

The MNIST dataset is augmented by (1) randomCrop: randomly crop each image with size 28 and padding 4, (2) random rotation: randomly rotate each image by a degree in $[-15^{\circ}, +15^{\circ}]$, (3) normalization: normalize each image with empirical mean 0.1307 and standard deviation 0.3081.

The FMNIST dataset is augmented by (1) randomCrop: same as MNIST, (2) randomHorizontalFlip: randomly flip each image horizontally, (3) normalization: same as MNIST, (4) random erasing \citep{Zhong17}.

The CIFAR10 dataset is augmented by (1) randomCrop: randomly crop each image with size 32 and padding 4, (2) randomHorizontalFlip: randomly flip each image horizontally, (3) normalization: normalize each image with empirical mean (0.4914, 0.4822, 0.4465) and standard deviation (0.2023, 0.1994, 0.2010),  (4) random erasing.

\end{document}